\documentclass[11pt]{article}
\usepackage{amsthm}
\usepackage{chicagor}

\usepackage{graphicx}
\usepackage{algorithmic}
\usepackage{algorithm}
\usepackage{amsmath}
\usepackage{amssymb}
\usepackage{mathrsfs}

\usepackage{framed}
\usepackage{wrapfig,lipsum,booktabs}

\newcommand{\shortv}{\commentout}
\newcommand{\fullv}[1]{#1}

\DeclareMathOperator*{\argmin}{argmin}

\DeclareSymbolFont{AMSb}{U}{msb}{m}{n}
\DeclareMathSymbol{\N}{\mathbin}{AMSb}{"4E}
\DeclareMathSymbol{\Z}{\mathbin}{AMSb}{"5A}
\DeclareMathSymbol{\R}{\mathbin}{AMSb}{"52}
\DeclareMathSymbol{\Q}{\mathbin}{AMSb}{"51}
\DeclareMathSymbol{\I}{\mathbin}{AMSb}{"49}
\DeclareMathSymbol{\C}{\mathbin}{AMSb}{"43}

\newcommand{\cP}{\mathcal{P}}
\newcommand{\cD}{\mathcal{D}}

\newcommand{\commentout}[1]{}
\newcommand{\ucP}{\overline{\cP}^+}

\newcommand{\regret}{\mathit{reg}}

\newcommand{\inter}{\cap}

\shortv{\renewcommand{\citeyear}{\cite}}
\newtheorem{axiom}{Axiom}
\fullv{
\newtheorem{theorem}{Theorem}[section]
\newtheorem{definition}[theorem]{Definition}

\newtheorem{lemma}[theorem]{Lemma}
\newtheorem{claim}[theorem]{Claim}
\newtheorem{proposition}[theorem]{Proposition}
\theoremstyle{definition}
\newtheorem{example}[theorem]{{Example}}
}
\DeclareSymbolFont{AMSb}{U}{msb}{m}{n}

\DeclareMathSymbol{\N}{\mathbin}{AMSb}{"4E}
\DeclareMathSymbol{\Z}{\mathbin}{AMSb}{"5A}
\DeclareMathSymbol{\R}{\mathbin}{AMSb}{"52}
\DeclareMathSymbol{\Q}{\mathbin}{AMSb}{"51}
\DeclareMathSymbol{\I}{\mathbin}{AMSb}{"49}
\DeclareMathSymbol{\C}{\mathbin}{AMSb}{"43}

\fullv{

}

\newcommand{\<}{\langle}
\renewcommand{\>}{\rangle}
\newcommand{\wbox}{\mbox{$\sqcap$\llap{$\sqcup$}}}

\begin{document}

\fullv{
\title{Minimizing Regret in Dynamic Decision Problems%
\thanks{Work supported in part by NSF grants
IIS-0812045, IIS-0911036, and CCF-1214844, by AFOSR grants
FA9550-08-1-0438, FA9550-09-1-0266, and FA9550-12-1-0040, and by ARO
grant W911NF-09-1-0281.}}
\author{
{Joseph Y. Halpern
\ \ \ \ \ \ \ \ Samantha Leung }\\
Department of Computer Science\\
Cornell University\\
Ithaca, NY 14853 \\
halpern $|$ samlyy@cs.cornell.edu}
}
\shortv{
\author{
Joseph Y. Halpern  \and Samantha Leung}
\institute{Cornell University, Ithaca, NY 14853}

\frontmatter          
\pagestyle{headings}  
\mainmatter              
\title{Minimizing Regret in Dynamic Decision Problems}
}
\maketitle              

\begin{abstract}
The menu-dependent nature of regret-minimization creates subtleties when it is applied to dynamic decision problems.
It is not clear whether forgone opportunities should be included in the menu.
We explain commonly observed behavioral patterns as minimizing regret
when forgone opportunities are present.
If forgone opportunities are included, we can characterize
when a form of dynamic consistency is guaranteed.

\end{abstract}

\section{Introduction}

Savage \citeyear{Savage1951} and Anscombe and Aumann
\citeyear{AnscombeAumann1963} showed that a decision maker maximizing
expected utility with respect to a probability measure over the
possible states of the world is characterized by a set of arguably
desirable principles.
However, as Allais \citeyear{Allais1953} and Ellsberg \citeyear{Ellsberg1961}
point out using compelling examples, sometimes intuitive choices are
{incompatible} with maximizing expected utility.
One reason for this incompatibility is that there is often
\emph{ambiguity} in the problems we face;
we often lack sufficient information to capture all
uncertainty using a single probability measure\shortv{.} \fullv{over the possible
states.} 

To this end, there is a rich literature offering alternative means of
making decisions
(see, e.g., \cite{Weinstein2009} for a survey).
For example, we might choose to represent uncertainty using a set of possible states of the world, but using no probabilistic information at all to represent how likely each state is.
With this type of representation, two well-studied rules for decision-making are \emph{maximin utility} and \emph{minimax regret}.
Maximin says that you should choose the option that maximizes the worst-case payoff, while minimax regret says that you should choose the option that minimizes the \emph{regret} you'll feel at the end, where, roughly speaking, regret is the difference between the payoff you achieved, and the payoff that you could have achieved had you known what the true state of the world was.
Both maximin and minimax regret can be extended naturally to deal
with other representations of uncertainty. For
example, with a set of probability measures over the possible states,
minimax regret becomes minimax expected regret (MER)
\cite{Hayashi2011,StoyeRegret}.
Other works that use a set of probablity measures include, for
example, \cite{CM95,CMW99,GilboaSchmeid93,Levi85,Walley91}. 

In this paper, we consider a generalization of minimax expected regret
called minimax \emph{weighted} expected regret (MWER) that
we introduced in an earlier paper \cite{HalpernLeung2012}.
For MWER, uncertainty is represented by a
set of \emph{weighted} probability measures.  Intuitively,
the weight represents how likely the
probability measure is to be the true distribution over the states,
according to the decision maker (henceforth DM).
The weights work much like a ``second-order'' probability on
the set of probability measures. 
Similar ideas can be dated back to at least G\"ardenfors and Sahlin
\citeyear{GS82,GS83}; see also \cite{Good80} for discussion and
further references. 
Walley \citeyear{Walley97} suggested putting a possibility measure
\cite{DP98,Zadeh1} on probability measures; this was also essentially
done by Cattaneo \citeyear{Cattaneo07}, Chateauneuf and Faro
\citeyear{ChateauneufFaro2009}, and de Cooman \citeyear{Cooman05}.  
All of these authors and others (e.g., Klibanoff et
al. \citeyear{KMM05}; Maccheroni et al. \citeyear{MMR06}; Nau
\citeyear{Nau92}) proposed approaches to decision making using their
representations of uncertainty. 

Real-life problems are often {dynamic}, with many stages where actions
can be taken; information can be learned over time.
Before applying regret minimization to dynamic decision problems, there is a
subtle issue that we must consider.
In static decision problems, the regret for each act is computed with
respect to a \emph{menu}.
That is, each act is judged against the other acts in the menu.
Typically, we think of the menu as consisting of the \emph{feasible
acts}, that is, the ones that the DM can perform.
The analogue in a dynamic setting would be the feasible \emph{plans},
where a plan is just a sequence of actions leading to a final outcome.
In a dynamic decision problem,
as more actions are taken, some plans become \emph{forgone opportunities}.
These are plans that were initially available to the DM,
but are no longer available due to earlier actions of the DM.
Since regret intuitively captures comparison of a choice against its
alternatives, it seems reasonable for the menu to include all the
feasible plans at the point of decision-making.
But should the menu include forgone opportunities?

\emph{Consequentialists} would
argue that it is irrational to care about forgone opportunities \cite{Hammond76,Machina1989}; we
should simply focus on the opportunities that are still available to
us, and thus not include forgone opportunities in the menu. 
And, indeed, when regret has been considered in dynamic settings thus
far (e.g., by Hayashi \citeyear{Hayashi2011}), the menu has not included
forgone opportunities.
However, introspection tells us that we sometimes do take forgone
opportunities into account\shortv{.} \fullv{when we feel regret.}  For example, when
considering
a new job, one might compare the available options to what
might have been available if one had chosen a different career path years
ago.
As we show, including forgone opportunities in the menu can make a big
difference in behavior.
%
Consider
 procrastination: we tell ourselves
that we will start studying for an exam (or start exercising, or quit
smoking) tomorrow; and then tomorrow comes, and we again tell
ourselves that we will do it, starting tomorrow.
This behavior is hard to explain with standard decision-theoretic
approaches, especially when we assume that no new information about
the world is gained over time.
However, we give an example where, if forgone opportunities are not included in
the menu, then we get procrastination; if they are, then we do not get
procrastination.

This example can be generalized.  Procrastination is an example of
\emph{preference reversal}: the DM's preference at time $t$ for what
he should do at time $t+1$ reverses when she actually gets to time
$t+1$.  We prove in Section~\ref{sec:FO}
that if the menu includes forgone opportunities and
the DM acquires no new information over time (as is the case in the
procrastination problem), then a DM who uses regret to make her
decisions will not suffer preference reversals.
Thus, we arguably get more rational behavior when we include forgone
opportunities in the menu.

What happens if the DM does get information over time?  It is well
known that, in this setting, expected utility maximizers are
guaranteed to have no preference reversals. Epstein and
Le Breton \citeyear{Epstein1993} have shown that, under minimal
assumptions, to avoid preference reversals, the DM must be an
expected utility maximizer.  On the other hand, Epstein and Schneider
\citeyear{EpsteinSchneider2003} show that a DM using MMEU 
never has preference reversals if her beliefs satisfy a condition they
call \emph{rectangularity}.
Hayashi \citeyear{Hayashi2011} shows that rectangularity also prevents
preference reversals for MER under certain assumptions.
Unfortunately, the rectangularity condition is often not satisfied in practice.
Other conditions have been provided that guarantee dynamic consistency
for ambiguity-averse decision rules (see, e.g., \cite{Weinstein2009}
for an overview).

We consider the question of
{preference reversal}
 in the context of regret.
Hayashi \citeyear{Hayashi2011} has observed that, in dynamic decision
problems, both
changes in menu over time and updates to the DM's beliefs can result
in preference reversals.
In Section~\ref{sec:consistence}, we show that keeping forgone
opportunities in the menu is necessary in order
to prevent preference reversals.
But, as we show by example, it is not sufficient if
the DM acquires new information over time.
We then provide a condition
on the beliefs that is necessary and sufficient to guarantee that a DM making decisions
using MWER whose beliefs satisfy the condition
will not have preference reversals.
However, because this necessary and sufficient condition may not be easy to
check, we also give simpler sufficient condition,
similar in spirit to Epstein and Schneider's
\citeyear{EpsteinSchneider2003}
rectangularity condition.
\fullv{Since MER can be understood as a special case of MWER where all
weights are either $1$ or $0$, our condition for dynamic consistency
is also applicable to MER.
}

\commentout{
What can be done if the DM's beliefs do not satisfy our condition
and dynamic inconsistency occurs?
A standard approach to avoiding dynamic inconsistency is \emph{sophistication}.
A sophisticated DM is one that understands and plans
around her future preferences,
instead of naively believing that her future selves will carry out the
plan that she currently prefers.
A huge literature considers sophisticated decision
makers, going back to the work of Strotz \citeyear{Strotz1955}.
But once we consider sophistication in the context of regret, we need
to reconsider the issue of what the appropriate menu is.
Assuming sophistication affects what plans should be counted as
feasible.  In particular, sophistication leads to \emph{unachievable plans},
plans that cannot be
carried out by the DM because future choices specified by
this plan involve choices that her future selves are unwilling to
make.
The impact of unachievable plans (and forgone
opportunities) on the menu has not been considered in previous work on
regret in a dynamic setting with sophisticated decision makers
\citeyear{Hayashi2009,KrahmerStone2005}.
}

\fullv{
The remainder of the paper is organized as follows.
Section~\ref{sec:MWER} discuss preliminaries.
Section~\ref{sec:FO} introduces forgone opportunities.
Section~\ref{sec:consistence} gives conditions under which consistent planning is not required.
We conclude in Section~\ref{sec:conclusion}.
We defer most proofs to the appendix.
}
\shortv{
\vspace{-5pt}
}
\section{Preliminaries}\label{sec:MWER}

\subsection{Static decision setting and regret}
Given a set $S$ of states and a set $X$ of outcomes,
an \emph{act} $f$ (over $S$ and $X$) is a function mapping $S$ to $X$.
We use $\mathcal{F}$ to denote the set of all acts.
For simplicity in this paper, we take $S$ to be finite.
Associated with each outcome $x\in X$ is a utility: $u(x)$ is the
utility of outcome $x$.  We call a tuple $(S,X,u)$ a
\emph{(non-probabilistic) decision problem}.
To define regret, we need to assume that we are also given a set $M\subseteq \mathcal{F}$ of acts, called the \emph{menu}.
The reason for the menu is that, as is well known, regret can depend on the menu. 
We assume that every menu $M$ has utilities bounded from
above.  That is,
we assume that for all menus $M$, $\sup_{g\in M}u(g(s))$
is finite.  This ensures that the regret of each act is well defined.
For a menu $M$ and act $f\in M$, the regret of $f$ with respect to
$M$ and decision problem
$(S,X,u)$
in state $s$ is
\fullv{
$$\regret_M(f,s) = \left(\sup_{g\in M}u(g(s))\right) - u(f(s)).$$
}
\shortv{
$\regret_M(f,s) = \left(\sup_{g\in M}u(g(s))\right) - u(f(s)).$
}
That is, the regret of $f$ in state $s$ (relative to menu $M$) is the
difference between $u(f(s))$ and the highest utility possible in state
$s$ among all the acts in $M$.
The regret of $f$ with respect to $M$ and decision problem $(S,X,u)$,
denoted $\regret_M^{(S,X,u)}(f)$, 
is
the worst-case regret over all states:
\fullv{
$$
\regret_M^{(S,X,u)}(f) = \max_{s\in S}\regret_M(f,s).
$$
}
\shortv{
$\max_{s\in S}\regret_M(f,s).$
}
We typically omit 
superscript $(S,X,u)$ in $\regret_M^{(S,X,u)}(f)$if it is clear from context.
The minimax regret decision rule chooses an act that minimizes
$\max_{s\in S}\regret_M(f,s).$ 
In other words, the minimax regret choice function is
\fullv{
$$C_{M}^{\regret}(M') = \argmin_{f \in M'} \max_{s\in S}\regret_M(f,s).$$
}
\shortv{
$C_{M}^{\regret,u}(M') = \argmin_{f \in M'} \max_{s\in S}\regret_M(f,s);$
}
\shortv{the}\fullv{The} choice function 
returns the set of all acts in $M'$ that minimize regret with respect
to $M$.
Note that we allow the menu $M'$, the set of acts over which we are
minimizing regret, to be different from the menu $M$ of acts with
respect to which regret is computed.  
For example, if the DM considers forgone opportunities, they would be
included in $M$, although not in $M'$.

If there is a probability measure $\Pr$ over the 
$\sigma$-algebra $\Sigma$ on the set $S$ of
states, then we can
consider the \emph{probabilistic decision problem} $(S,\Sigma, X,u,\Pr)$.
The \emph{expected regret} of $f$ with respect to $M$ is
\fullv{
$$
\regret_{M}^{\Pr}(f) =
\sum_{s\in S}\Pr(s)\regret_M(f,s).
$$
}
\shortv{
$
\regret_{M}^{\Pr}(f) =
\sum_{s\in S}\Pr(s)\regret_M(f,s).
$
}
If there is a set $\cP$ of probability measures over the
$\sigma$-algebra $\Sigma$ on the set $S$ of states,
states, then we consider the $\cP$-decision problem $\cD = (S,\Sigma,X,u,\cP)$.
The maximum expected regret of $f\in M$ with respect to $M$ and
$\cD$ is
\fullv{
$$
\regret_{M}^{\cP}(f) = \sup_{\Pr\in \cP} \left( \sum_{s\in S}\Pr(s)
\regret_M(f,s)   \right).
$$
}
\shortv{
$
\regret_{M}^{\cP}(f) = \sup_{\Pr\in \cP} \left( \sum_{s\in S}\Pr(s)
\regret_M(f,s)   \right).
$
}
\noindent The minimax expected regret (MER) decision rule minimizes
$\regret_{M}^{\cP}(f)$.

In an earlier paper, we introduced another representation of
uncertainty, \emph{weighted set of probability measures}
\cite{HalpernLeung2012}.
A weighted set of probability measures generalizes a set of probability measures by associating each measure in the set with a weight, intuitively corresponding to the reliability or significance of the measure in capturing the true uncertainty of the world.
Minimizing {weighted} expected regret with respect to a weighted set of probability measures gives a variant of minimax regret, called Minimax Weighted Expected Regret (MWER).
A set $\cP^+$ of \emph{weighted
probability measures} on $(S,\Sigma)$ consists of pairs
$(\Pr,\alpha_{\Pr})$, where $\alpha_{\Pr} \in [0,1]$ and
$\Pr$ is a probability measure on $(S,\Sigma)$.
Let $\cP = \{\Pr: \exists \alpha (\Pr,\alpha) \in \cP^+\}$.
We assume that,
for each $\Pr \in \cP$, there is exactly one $\alpha$ such that
$(\Pr,\alpha) \in \cP^+$.  We denote this number by $\alpha_{\Pr}$, and view
it as the \emph{weight of $\Pr$}.
We further assume for convenience that weights have
been normalized so that there is at least one
measure $\Pr \in \cP$ such that $\alpha_{\Pr} = 1$.%

If beliefs are modeled by a set $\cP^+$ of weighted probabilities, then
we consider the $\cP^+$-decision problem $\cD^+ = (S,X,u,\cP^+)$.
The maximum weighted expected regret of $f\in M$ with respect to
$M$ and $\cD^+ = (S,X,u,\cP^+)$ is
$$
\regret_{M}^{\cP^+}(f)
= \sup_{(\Pr,\alpha) \in \cP^+}\left( \alpha \sum_{s\in
S}\Pr(s)\regret_M(f,s) \right).
$$
If $\cP^+$ is empty, then $\regret_{M}^{\cP^+}$ is identically zero.
Of course, we can define the choice functions
$C^{\regret,\Pr}_{M}$, $C^{\regret,\cP}_{M}$, and $C^{\regret,\cP^+}_{M}$
using $\regret_{M}^{\Pr}$, $\regret_{M}^{\cP}$, and
$\regret_{M}^{\cP^+}$, by analogy with $C^{\regret}_M$.

\shortv{
\vspace{-5pt}
}
\subsection{Dynamic decision problems}

A \emph{dynamic decision problem} is a single-player extensive-form
game where there is some set $S$ of states, nature chooses $s \in S$ at the
first step, and does not make any more moves.
The DM then performs a finite sequence of actions until some
outcome is reached.
Utility is assigned to these outcomes.
A \emph{history} is a sequence recording the actions taken by nature
and the DM.
At every history $h$, the DM considers possible some other histories.
The DM's \emph{information set} at $h$, denoted $I(h)$, is the set of
histories that the DM considers possible at $h$.
Let $s(h)$ denote the initial state of $h$ (i.e., nature's first
move);
let $R(h)$ denote all the moves the DM made in $h$ after
nature's first move; finally, let $E(h)$ denote the set of states that the DM
considers possible at $h$; that is, $E(h) = \{s(h'): h' \in I(h)\}$.
We assume that the DM has \emph{perfect recall}:
this means that $R(h') = R(h)$ for all $h' \in I(h)$, and that if $h'$
is a prefix of $h$, then $E(h') \supseteq E(h)$.

A \emph{plan} is a (pure) strategy: a mapping from histories to
histories that result from taking the action specified by the plan.
We require that a plan specify the same action for all histories in an
information set; that is, if $f$ is a plan, then for all histories $h$
and $h' \in I(h)$, we must have the last action in $f(h)$ and  $f(h')$
must be the same (so that $R(f(h)) = R(f(h'))$).
Given an initial state $s$, a plan determines a complete path to an
outcome. Hence, we can also view plans as acts: functions
mapping states to outcomes.
We take the acts in a dynamic decision problem to be the set of possible plans, and
evaluate them using the decision rules
discussed above.

A major difference between our model and that used by Epstein and
Schneider \citeyear{EpsteinSchneider2003} and Hayashi
\citeyear{Hayashi2009} is that the latter assume
a \emph{filtration} 
information structure.
With a filtration information structure, the DM's knowledge is
represented by a fixed, finite sequence of partitions.
More specifically, at time $t$, the DM uses a partition $F(t)$ of
the state space, and if the true state is $s$, then all that the DM
knows is that the true state is in the cell of $F(t)$ containing $s$.
Since the sequence of partitions is fixed, the DM's knowledge
is independent of the choices that she makes,
and her options and preferences cannot depend on past choices.
This assumption significantly restricts the
types of problems that can be naturally modeled.
For example, if the DM prefers to have one apple over two oranges at
time $t$, then this must be her time $t$ preference, regardless of
whether she has already consumed five apples at time $t-1$.
Moreover, consuming an apple at time $t$ cannot preclude consuming an
apple at time $t+1$.
Since we effectively represent a decision problem as a single-player
extensive-form game, we can capture all of these situations in a
straightforward way.
The models of Epstein, Schneider, and Hayashi can be viewed
as a special case of our model.
\commentout{
Moreover,
Epstein, Schneider, and Hayashi focus on the ranking of different
choices at different times.
This focus has two main consequences.
Firstly, the DM's preferences at a certain time are assumed to be
independent of the choices made at different times.
Secondly, the model has no specification of the actual choice set at
time $t$, since the focus is only on the ranking between different
acts.
For example, while the DM may prefer one apple over two oranges at
time $t$, there is no specification on what bundles are actually
available to be chosen at time $t$.
Furthermore, the model does not capture the structure of a dynamic
decision problem, e.g., the fact that consuming an apple at time $t$
will preclude the opportunity of having apple juice at time $t+1$.
This modeling assumption is different from our pragmatic approach that represent a decision problem using a single-player extensive-form game.
In fact, the models of Eptein, Schneider, and Hayashi can be thought
of as special cases of our model, where the information structure is
fixed and independent of the DM's choices, and where the decision tree
has a structure such that past actions do not affect current and
future choices.
}

In a dynamic decision problem, as we shall see, two different menus
are relevant for
making a decision using regret-minimization: the menu with respect to
which regrets are computed, and the menu of feasible choices.
We formalize this dependence by considering \emph{choice functions}
of the form
$C_{M,E}$, where $E, M \ne \emptyset$.  $C_{M,E}$ is a function
mapping a nonempty menu $M'$
to a nonempty subset of $M'$.  Intuitively, $C_{M,E}(M')$ consists of
the DM's most preferred choices
from the menu $M'$ when she considers the
states in $E$ possible and
her decision are made relative to menu $M$.  (So, for example, if the
DM is making her choices choices using regret minimization, the regret
is taken with respect to $M$.)
Note that there may be more than one plan in $C_{M,E}(M')$;
intuitively, this means that the DM does not view any of the plans in
$C_{M,E}(M')$ as strictly worse than some other plan.

What should $M$ and $E$ be when the DM makes a decision at
a history $h$?  We always take $E = E(h)$.  Intuitively, this says
that all that matters about a history as far as making a decision is
the set of states that the DM considers possible; the previous moves
made to get to that history are irrelevant.  As we shall see, this
seems reasonable in many examples.  Moreover, it is consistent with
our choice of taking probability distributions only on the state space.

The choice of $M$ is somewhat more subtle.  The most obvious choice
(and the one that has typically been made in the literature, without
comment) is that $M$ consists of the plans that are still
feasible at $h$, where
a plan $f$ is \emph{feasible} at a history $h$ if,
for all strict prefixes $h'$ of $h$, $f(h')$ is also a
prefix of $h$.  So $f$ is feasible at $h$ if $h$ is compatible with
all of $f$'s moves.  Let $M_h$ be the set of plans feasible at $h$.
While taking $M=M_h$ is certainly a reasonable choice, as we shall
see, there are other reasonable alternatives.

Before addressing the choice of menu in more detail, we consider how
to apply regret
in a dynamic setting.
If we want to apply MER or MWER, we must update the probability
distributions.
Epstein and Schneider \citeyear{EpsteinSchneider2003} and Hayashi \citeyear{Hayashi2009}
consider \emph{prior-by-prior updating}, the most common way to update
a set of probability measures, 
defined as follows:
\fullv{$$
\cP|^p E = \{\Pr | E: \Pr \in \cP, \Pr(E) > 0\}.
$$}
\shortv{$\cP|^p E = \{\Pr | E: \Pr \in \cP, \Pr(E) > 0\}.$}
\fullv{ We can also apply prior-by-prior updating to a weighted set of probabilities:
$$\cP^+|^p E = \{(\Pr | E, \alpha ): (\Pr,\alpha) \in \cP^+, \Pr(E) > 0\}.$$
}

Prior-by-prior updating can produce some rather
counter-intuitive outcomes.
For example, suppose we have a coin of unknown bias in $[0.25,0.75]$,
and flip it $100$ times. 
We can represent our prior beliefs using a set of probability measures.
However, if we use prior-by-prior updating, then after each flip of the coin
the set $\cP^+$ representing the DM's beliefs does not change, 
because the beliefs are independent.
Thus, in this example, prior-by-prior updating is not capturing the information provided by the flips.

We consider another way of updating weighted sets of
probabilities, called \emph{likelihood updating}
\cite{HalpernLeung2012}.
The intuition is that the weights are updated as if they were a
second-order probability distribution over the probability measures.
Given an event $E\subseteq S$,
define $\ucP(E) = \sup\{\alpha\Pr(E): (\Pr,\alpha) \in \cP^+\}$;
if $\ucP(E) > 0$, let $\alpha^l_{E}= \sup_{\{(\Pr',\alpha') \in\cP^+ :
\Pr'| E = \Pr|E\}} \frac{\alpha'\Pr'(E)}{\ucP(E)}$.  Given a
measure  $\Pr \in \cP$, there may be
several distinct measures $\Pr'$ in $\cP$ such that $\Pr'| E = \Pr
| E$.  Thus, we take the weight of $\Pr | E$ to be the $\sup$ of the
possible candidate values of  $\alpha^l_{E}$.  By dividing by
$\ucP(E)$, we guarantee that $\alpha^l_{E} \in [0,1]$, and that
there is some weighted measure $(\Pr,\alpha)$ such that $\alpha^l_{E} = 1$,
as long as there is some pair $(\Pr,\alpha) \in \cP^+$ such that
$\alpha \Pr(E) = \ucP(E)$.
If $\ucP(E) > 0$,
we take $\cP^+ |^l E$, the result of applying likelihood updating by
$E$ to $\cP^+$, to be
$$\{(\Pr | E, \alpha^l_{E} ): (\Pr,\alpha) \in \cP^+, \Pr(E) > 0\}.$$ 

In computing $\cP^+ |^l E$, we update not just the probability
measures in $\Pr \in \cP$, but also their weights, which are updated to
$\alpha_{E}^{l}$. 
Although prior-by-prior updating does not change the weights, for
purposes of exposition,
given a weighted probability measure $(\Pr,\alpha)$, 
we use $\alpha_{E}^{p}$ to denote the
``updated weight'' of $\Pr| E \in \cP^+|^p E$; of course,
$\alpha_{E}^{p} = \alpha$.

Intuitively, probability measures that are supported by the new
information will get larger weights 
using likelihood updating
than those not supported by the new information.
Clearly, if all measures in $\cP$
start off with the same weight and
assign the same probability to the
event $E$, then likelihood updating will give the same weight to each probability measure, resulting in measure-by-measure
updating.
This is not surprising, since such an observation $E$ does not give us
information about the relative likelihood of measures.
\commentout{
Using likelihood updating is appropriate only if
the measure generating the observations is
assumed to be stable.
This is because if the generating probability distribution changes
with time, then we cannot converge to a single generating
distribution.
For example, if observations of heads and tails are generated by coin
tosses, and a coin of possibly different bias is tossed in each round,
then likelihood updating would not be appropriate.
It is not obvious what kind of updating should be done in this more
general setting.
}

Let $\regret^{\cP^+|^l E}_M(f)$ denote the regret of act $f$ computed
with respect to menu $M$ and beliefs $\cP^+|^l E$.   
If $\cP^+|^l E$ is empty 
(which will be the case if $\ucP(E) = 0$) 
then $\regret^{\cP^+|^l E}_M(f) = 0$ for all acts $f$.
We can similarly define $\regret^{\cP^+|^p E}_M(f)$ for beliefs updated using prior-by-prior updating.
Also, let $C^{\regret,\cP^+|^l E}_{M}(M')$ be the set of acts in $M'$ that minimize the weighted expected regret $\regret^{\cP^+|^l E}_M$.
If $\cP^+|^l E$ is empty, then $C^{\regret,\cP^+|^l E}_{M}(M') = M'$. 
We can similarly define \fullv{$C^{\regret,\cP^+|^p E}_{M}$,}
$C^{\regret,\cP|E}_{M}$ and $C^{\regret,\Pr|E}_{M}$.

\shortv{
\vspace{-5pt}
}
\section{Forgone opportunities}\label{sec:FO}

\shortv{
\begin{wrapfigure}{r}{0.50\textwidth}
\vspace{-60pt}
\includegraphics[width=0.48\textwidth]{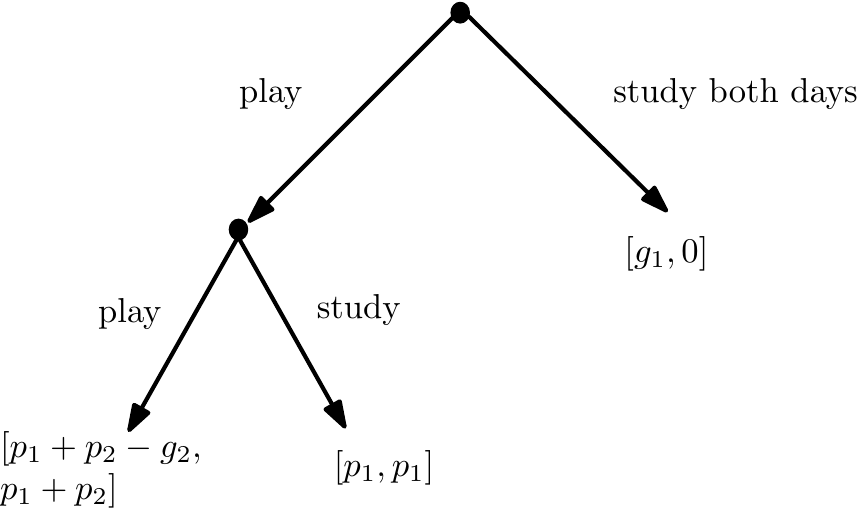}
	\caption{An explanation for procrastination.}\label{fig:procrastination}
\vspace{-20pt}
\end{wrapfigure}
}
As we have seen, when making a decision at a history $h$ in a dynamic
decision problem, the DM must decide what menu to use.  In this
section we focus on one choice.  Take a \emph{forgone opportunity} to
be a plan that was initially available to the DM, but is no longer
available due to earlier actions.  As we observed in the introduction,
while it may seem irrational to consider forgone opportunities, people
often do.  Moreover, when combined with regret, behavior that results
by considering forgone opportunities may be arguably \emph{more}
rational than if forgone opportunities are not considered.
Consider the following example.
\begin{example}\label{xam:student}
Suppose that
a student has an exam in two days.
She can either start studying today, play today and then study tomorrow, or just play on both days and never study.
There are two states of nature: one where the exam
is difficult, and one where the exam is easy.
The utilities reflect a combination of the amount of pleasure that the
student derives in the next two days, and her score on the exam
relative to her classmates.
Suppose that the first day of play gives the student $p_1>0$ utils,
and the second day of play gives her $p_2>0$ utils.
Her exam score affects her utility only in the case where the exam is
hard and she studies both days, in which case she gets an
additional $g_1$ utils for doing much better than everyone else, and
in the case where the exam is hard and she never studies, in which
case she loses $g_2 > 0$ utils for doing much worse than everyone
else.
Figure~\ref{fig:procrastination} provides a graphical representation
of the decision problem.
Since, in this example, the available actions for the DM
are independent of
nature's move,
for compactness, we omit nature's initial move (whether the exam is
easy or hard).
Instead, we describe the payoffs of the DM as a pair $[a_1,a_2]$,
where $a_1$ is the payoff if the exam is hard, and $a_2$ is the payoff
if the exam is easy.
\fullv{
\begin{figure}[htb]
	\centering
		\includegraphics[width=0.48\textwidth]{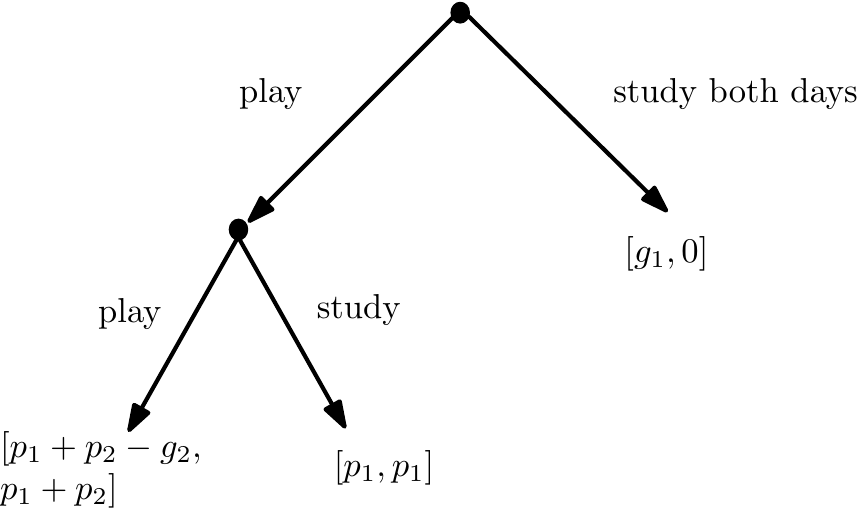}
			\caption{An explanation for procrastination.}\label{fig:procrastination}
\end{figure}
}

Assume that $2p_1 + p_2 > g_1 > p_1 + p_2$ and $2p_2 > g_2 > p_2$.
That is, if the test were hard, the student would be happier studying and doing well on the test
than she would be if she played for two days,
but not too much happier; similarly, the penalty for doing badly in
the exam if the exam is hard and she does not study is greater than
the utility of playing the second day, but not too much greater.
Suppose that the student uses minimax regret to make her decision.
On the first day, she observes that playing one day and then
studying the next day has a worst-case regret of $g_1 - p_1$, while
studying on both days has a worst-case regret of $p_1 + p_2$.
Therefore, she plays on the first day.
On the next day,
suppose that she does not consider forgone opportunities and just
compares her two available options, studying and playing.
Studying has a worst-case regret of $p_2$, while playing has a
worst-case regret of $g_2-p_2$, so, since $g_2 < 2p_2$,
she plays again on the second day.
On the other hand, if the student had included the forgone opportunity
in the menu on the second day, then studying would have regret $g_1 -
p_1$, while playing would have regret $g_1 + g_2 - p_1 - p_2$.
Since $g_2 > p_2$, studying minimizes regret.
\hfill \wbox
\end{example}
Example~\ref{xam:student} emphasizes the roles of the menus $M$ and
$M'$ in $C_{M,E}(M')$. Here we took $M$, the menu relative to which
choices were evaluated, to consist of all plans, even
the ones that were no longer feasible, while $M'$ consisted of only
feasible plans.
In general, to determine the menu component $M$ of the choice function
$C_{M,E(h)}$ used at a history $h$, we use a \emph{menu-selection
  function} $\mu$.  The menu $\mu(h)$ is the menu relative to which
choice are computed at $h$.
We sometimes write $C_{\mu,h}$ rather than $C_{\mu(h),E(h)}$.

We can now formalize the notion of \emph{no preference reversal}.
\noindent
\commentout{
There is \emph{no preference reversal} if, for all histories $h$ and
$h'$ such that $h$ precedes $h'$ and all plans $f$, if
is one of the best plans at $h$ and it is still feasible at history $h'$,
then $f$ is one of the best plans at $h'$; that is, if $f \in
C_{\mu,E(h)}(M(h))$ and $f \in M(h')$, then $f \in C_{\mu,E(h')}(M(h'))$.}
\commentout{No preference reversal implies that any plan initially considered
among the best plans will
be one of the best plans whenever it is feasible.
If the DM moves according to the plan starting from the beginning,
then the plan always remains feasible.
Therefore, the DM will be able to carry out the plan,
even if we require that, at each step the DM always takes the first
step of a plan currently considered among the best.
}
\fullv{Roughly speaking, this says that if a plan $f$ is considered one of
  the best at history $h$ and is still feasible at an extension $h'$
  of $h$, then $f$ will still be considered one of the best plans at $h'$.}

\commentout{
\begin{definition}
A plan $f$ is \emph{ex ante optimal} if $f \in
C_{\mu,\langle s\rangle}(M_{\langle s\rangle})$, where $s$ is the initial state chosen by nature
(i.e., nature's first move).
\end{definition}
}

\shortv{
\vspace{-5pt}
}
\begin{definition}[No preference reversal]
A family of choice functions $C_{\mu, h}$ has \emph{no preference reversals}
if, for all histories $h$ and all histories $h'$ extending 
$h$, if $f \in C_{\mu,h}(M_h)$ 
and $f \in M_{h'}$, then $f \in C_{\mu, h'}(M_{h'})$.
\end{definition}

The fact that we do not get a preference reversal in
Example~\ref{xam:student} if we take forgone
opportunities into account here is not just an artifact of this example.
As we now show, as long as we do not get new information and also use
a constant menu (i.e., by keeping all forgone opportunities in the
menu),
then there will be no preference reversals if we
minimize (weighted) expected regret in a dynamic setting.

\begin{proposition}
\label{pro:noprefrev}
If, for all histories $h, h'$, we have $E(h) = S$ and $\mu(h) =
\mu(h')$, and decisions are made according to MWER
(i.e., the agent has a set $\cP^+$ of weighted probability
distributions and a utility function $u$, and
$f \in C_{\mu,h}(M_h)$ if $f$ minimizes weighted expected regret with respect
to $\cP^+|^lE(h)$\fullv{ or $\cP^+|^p E(h)$}),
then no preference reversals occur.
\end{proposition}
\fullv{
\begin{proof}
Suppose that $f \in C_{\mu,\langle s \rangle}$, $h$
is a history
  extending $\langle s \rangle$, and $f \in M_h$.
Since $E(h) = S$ and $\mu(h) = \mu(\langle s \rangle)$ by assumption,
we have $C_{\mu(h),E(h)} = C_{\mu(\langle s \rangle),E(\langle s \rangle)}$.
By assumption, $f \in C_{\mu(\<s\>, M_{\<s\>}}(M_{\<s\>}) =
C_{\mu(h),E(h)}(M_{\<s\>})$.
It is easy to check that MWER satisfies what is known in decision
theory as \emph{Sen's $\alpha$ axiom} \citeyear{Kreps}: if $f \in M'
  \subseteq M''$ and $f \in   C_{M,E}(M'')$, then $f \in C_{M,E}(M')$.
That is, if $f$ is among the most preferred acts in menu $M''$,
if $f$ is in the smaller menu $M'$, then it must also be among the most
preferred acts in menu $M'$.
Because $f \in M_h \subseteq M_{\langle s \rangle}$ and $f \in
C_{\mu,\langle s \rangle}(M_{\langle s \rangle})$, we
have $f \in C_{\mu(h),E(h)}(M_h)$, as required.
\end{proof}
}

\fullv{
\begin{table}
	\centering
		\begin{tabular}{|c|c|c|c|c|}\hline
		& \multicolumn{2}{|c|}{Hard} & \multicolumn{2}{|c|}{Easy} \\ \hline
		& Short & Long & Short & Long \\ \hline
		${\Pr}_1$ & 1 & 0 & 0 & 0 \\
		${\Pr}_2$ & 0 & $0.2$ & $0.2$ & $0.2$ \\ \hline
		play-study & 1 & 0 & 5 & 0 \\
		play-play & 0 & 3 & 0 & 3 \\ \hline
		\end{tabular}
		\caption{$\alpha_{\Pr_1} = 1, \alpha_{\Pr_2} = 0.6$.}
	\label{tab:reversalex}
\end{table}
}
\shortv{
\begin{wraptable}{l}{5.2cm}
\vspace{-20pt}
	\centering
		\begin{tabular}{|c|c|c|c|c|}\hline
		& \multicolumn{2}{|c|}{Hard} & \multicolumn{2}{|c|}{Easy} \\ \hline
		& Short & Long & Short & Long \\ \hline
		${\Pr}_1$ & 1 & 0 & 0 & 0 \\
		${\Pr}_2$ & 0 & $0.2$ & $0.2$ & $0.2$ \\ \hline
		play-study & 1 & 0 & 5 & 0 \\
		play-play & 0 & 3 & 0 & 3 \\ \hline
		\end{tabular}
		\caption{$\alpha_{\Pr_1} = 1, \alpha_{\Pr_2} = 0.6$.}
	\label{tab:reversalex}
\vspace{-40pt}
\end{wraptable}}Proposition~\ref{pro:noprefrev} shows that we cannot have preference
reversals if the DM does not learn about the world.
However, if the DM learns about the world, then we can have preference reversals.
Suppose, as is depicted in Table~\ref{tab:reversalex},
that in addition to being hard and easy, the exam can also be short or long.
The student's beliefs are described by the set of weighted
probabilities $\Pr_1$ and $\Pr_2$, with weights $1$ and
$0.6$, respectively.

We take the option of studying on both days out of the picture by
assuming that its utility is low enough for it to never be preferred, and for it to never affect the regret computations.
After the first day, the student learns whether the exam will be hard
or easy.
One can verify that the ex ante regret of playing then studying is
lower than that of playing on both days, while after the first day,
the student prefers to play on the second day, regardless of whether
she learns that the exam is hard or easy.

\shortv{
\vspace{-5pt}
}
\section{Characterizing no preference reversal}\label{sec:consistence}
\label{sec:dc}
\shortv{
\begin{wrapfigure}{r}{0.58\textwidth}
\vspace{-30pt}
\includegraphics[width=0.58\textwidth]{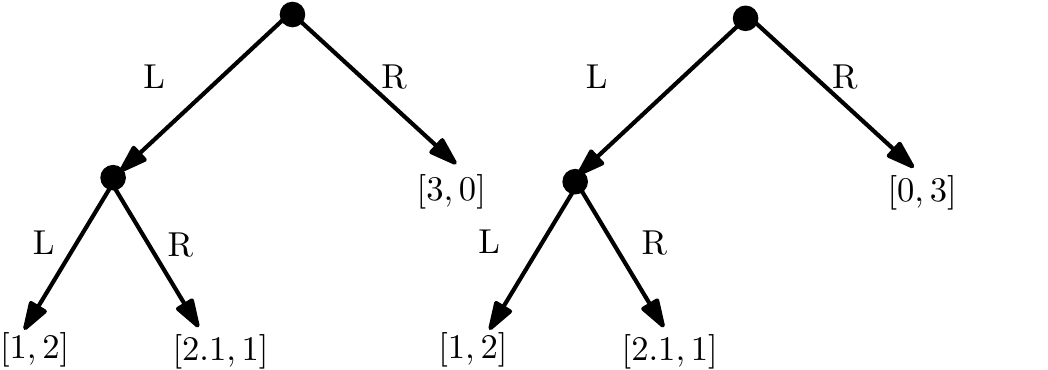}
	\caption{\label{fig:lostcause}}
\vspace{-20pt}
\end{wrapfigure}
}
We now consider conditions under which there is no preference reversal
in a more general setting, where
the DM can acquire new information.
While including all forgone opportunities is no longer a sufficient
condition to prevent preference reversals, it is necessary, as the
following example shows:
Consider the two similar decision problems depicted in
Figure~\ref{fig:lostcause}.
\fullv{
\begin{figure}[htb]
\vspace{-10pt}
	\centering
		\includegraphics[width=0.80\textwidth]{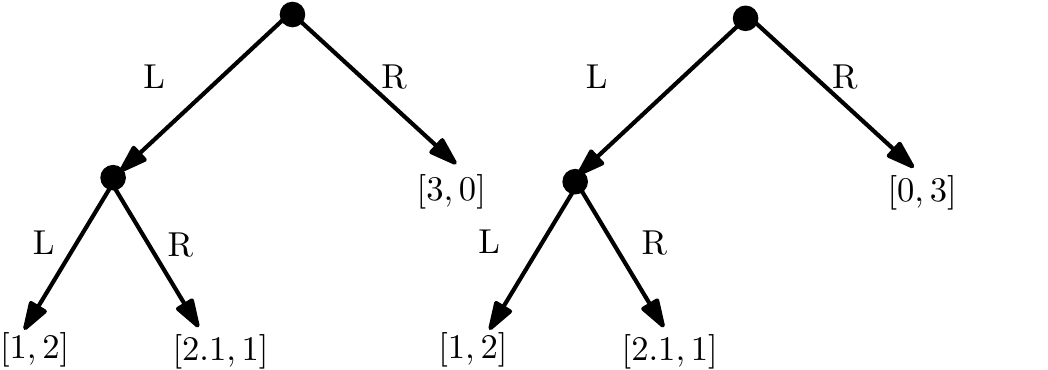}
	\caption{Two decision trees.}
	\label{fig:lostcause}
\vspace{-10pt}
\end{figure}
}
Note that at the node after first playing $L$, the utilities and
available choices are identical in the two problems.  If we ignore
forgone opportunities, the DM necessarily makes the same decision in
both cases if his beliefs are the same.
However, in the tree to the left, the ex ante optimal plan is $LR$, while in the tree to the right, the
ex ante optimal plan is $LL$.
If the DM ignores forgone opportunities, then after the first step, she cannot tell whether she is in the decision tree on the left side, or the one on the right side.
Therefore, if she follows the ex ante optimal plan in one of the trees, she
necessarily is not following the ex ante optimal plan in the other tree.

In light of this example, we now consider what happens if the DM
learns information over time.
Our no preference reversal condition is implied by a well-studied notion called \emph{dynamic consistency}.
One way of describing dynamic consistency is that a plan considered
optimal at a given point in the decision process is also optimal at
any preceding point in the process, as well as any future point
that is reached with positive probability \cite{Siniscalchi2011}.
For menu-independent preferences, dynamic consistency is
usually captured axiomatically by variations of an axiom called \emph{Dynamic
Consistency} (DC) or the \emph{Sure Thing Principle} \cite{Savage1954}.
We define a \emph{menu-dependent} version of DC relative to events $E$ and $F$ using the following axiom.
The second part of the axiom implies that if $f$ is strictly preferred
conditional on $E \cap F$ and at least weakly preferred on $E^c \cap F$, then $f$ is also strictly preferred on $F$.
An event $E$ is \emph{relevant to a dynamic decision problem $\cD$} if
it is one of the events that
the DM can potentially learn in $D$, that is,
if there exists a history $h$ such that $E(h) = E$.
A dynamic decision problem $\cD = (S,\Sigma,X,u,\cP)$ is ``proper'' if $\Sigma$ is generated by the subsets of $S$ relevant to $\cD$.
Given a decision problem $D$, we take the \emph{measurable sets} to be
the $\sigma$-algebra generated by the events relevant to $D$.  
The following axioms hold for all measurable sets $E$ and $F$, menus $M$ and $M'$, and acts $f$ and $g$.  
\shortv{\vspace{-5pt}}
\begin{axiom}[DC-M] \label{axiom:DC-M}
If $f \in C_{M_{},E \cap F }(M') \cap C_{M_{},E^c \cap F}(M')$, then $f\in
C_{M_{},F}(M')$.
If, furthermore, $g \notin C_{M_{},E \cap F }(M')$, then $g \notin C_{{M},F}(M')$.
\end{axiom}

\commentout{
If there exists a history $h$ such that $E(h) = E$,
 $f$ and $g$ are plans such that $f$ and $g$
agree at all histories $h$ where (1) $f$ is feasible and (2) $E(h)
\supseteq E$ or $E(h) \supseteq E^c$, then we let $fEg$ denote
the plan that follows the plan $f$ at all histories $h$ where $E(h)
\subseteq E$,
but follows the plan $g$ at all other histories;
otherwise $fEg$ is undefined.
If $fEg$ is defined, the $fEG$  gives the outcome of $f$ on states in
$E$ and the  outcome of $g$ on states in $E^c$.
}
\shortv{\vspace{-10pt}}
\begin{axiom}[Conditional Preference]\label{axiom:offE}
If $f$ and $g$, when viewed as acts, give the same
outcome on all states in $E$, then
$f \in C_{M,E}(M')$ iff $g \in C_{M,E}(M')$.
\end{axiom}
\fullv{ 
The next two axioms put some weak restrictions on choice functions.}
\shortv{\vspace{-10pt}}
\begin{axiom}\label{axiom:nonempty}
$C_{M,E}(M') \subseteq M'$ and
$C_{M,E}(M') \neq \emptyset$ if $M' \ne \emptyset$.
\end{axiom}
\shortv{\vspace{-10pt}}
\begin{axiom}[Sen's $\alpha$]\label{axiom:sens}
If $f \in C_{M,E}(M')$ and $M'' \subseteq M'$, then $f \in C_{M,E}(M'')$.
\end{axiom}

\begin{theorem}\label{prop:norev}
For a dynamic decision problem $D$, if
Axiom~\ref{axiom:DC-M}--\ref{axiom:sens} hold 
and $\mu(h)=M$ for some fixed menu $M$,
then there will be no preference reversals in $D$.
\end{theorem}

We next provide a representation theorem that characterizes
when Axioms~\ref{axiom:DC-M}--\ref{axiom:sens} hold for a MWER decision maker.
The following condition says that the unconditional regret can
be computed by separately computing the regrets conditional on 
measurable events 
$E \cap F$ and on $E^c \cap F$.
\begin{definition}[SEP]
The weighted regret of $f$ with respect to $M$ and $\cP^+$ is
\shortv{
\emph{separable} with respect to $|^l$ if for all measurable sets $E$
and $F$, 
\vspace{-3pt}
\begin{align*}
\regret_{M}^{\cP^+|^lF} (f) & =
\sup_{(\Pr,\alpha) \in \cP^+} \alpha \left( \Pr(E \cap F)
\regret_{M}^{\cP^+|^l(E \cap F)}(f)
+ \Pr(E^c \cap F) \regret_{M}^{\cP^+|^l(E^c \cap F)}(f)\right),
\end{align*}
\vspace{-3pt}
and if $\regret_{M}^{\cP^+|^l(E \cap F)}(f) \neq 0$, then
\begin{align}\label{equ:sep2}
\regret_{M}^{\cP^+|^lF} (f) & >  \sup_{(\Pr,\alpha) \in \cP^+} \alpha \Pr(E^c\cap F) \regret_{M}^{\cP^+|^l(E^c \cap F)}(f).
\end{align}
}
\fullv{
\emph{separable} with respect to $|^{\chi}$ ($\chi \in \{p,l\}$) if
for all measurable sets $E$ and $F$ such that  
$\ucP(E \cap F) > 0$ and $\ucP(E^c \cap F) > 0$,
\begin{align*}
\regret_{M}^{\cP^+|^{\chi}F} (f) & =
\sup_{(\Pr,\alpha) \in \cP^+} \alpha \left( \Pr(E \cap F)
\regret_{M}^{\cP^+|^{\chi}(E \cap F)}(f)
+ \Pr(E^c \cap F) \regret_{M}^{\cP^+|^{\chi}(E^c \cap F)}(f)\right),
\end{align*}
and if $\regret_{M}^{\cP^+|^{\chi}(E \cap F)}(f) \neq 0$, then
$$
\regret_{M}^{\cP^+|^{\chi}F} (f)  >  \sup_{(\Pr,\alpha) \in \cP^+} \alpha
\Pr(E^c\cap F) \regret_{M}^{\cP^+|^{\chi}(E^c \cap F)}(f). 
$$
}
\end{definition}

We now show that 
Axioms~\ref{axiom:DC-M}--\ref{axiom:sens} characterize SEP. 
Say that a decision problem $\cD$ is \emph{based on} $(S,\Sigma)$ if 
$\cD = (S,\Sigma,X,u,\cP)$ for some $X,u$, and $\cP$. 
In the following results, we will also make use of an alternative interpretation of weighted probability measures.
Define a \emph{subprobability measure} ${p}$ on $(S,\Sigma)$ to be like a
probability measure, in that it is a function mapping measurable subsets of
$S$ to $[0,1]$ such that ${p}(T \cup T') = {p}(T) +  {p}(T')$ for
disjoint sets $T$ and $T'$, except that it may not satisfy the
requirement that ${p}(S) = 1$.
We can identify a weighted probability distribution $(\Pr, \alpha)$
with the subprobability measure $\alpha \Pr$.  (Note that given a
subprobability measure ${p}$, there is a unique pair $(\alpha,\Pr)$
such that ${p} = \alpha \Pr$: we simply take $\alpha = {p}(S)$ and
$\Pr = {p}/\alpha$.)
Given a set $\cP^+$ of weighted probability measures, we let $C(\cP^+) = \{ p \geq \vec{0}: \exists c, \exists \Pr, (c,\Pr) \in \cP^+ \text{ and } p \leq c\Pr \}$.
\begin{theorem}\label{thm:equivalence}
If $\cP^+$ is a set of weighted distributions
on $(S,\Sigma)$ such that $C(\cP^+)$ is closed, then the following are
equivalent for $\chi \in \{p,l\}$: 
\begin{enumerate}
\item[(a)] For all decision problems $D$ based on $(S,\Sigma)$ and all menus
  $M$ in $D$, Axioms~\ref{axiom:DC-M}--\ref{axiom:sens} hold for
the family $C_M^{\regret, \cP^+|^{\chi}E}$ of choice functions.
\item[(b)] For all decision problems $D$ based on $(S,\Sigma)$, 
states $s \in S$, and acts $f\in M_{\< s \>}$, the
weighted regret of  
 $f$ with respect to $M_{\< s \>}$ and $\cP^+$ is separable with
respect to $|^{\chi}$. 
\end{enumerate}
\end{theorem}
\fullv{Note that Theorem~\ref{thm:equivalence} says that to check
that Axioms 1--4 hold, we 
need to check only that separability holds
for initial menus $M_{\<s\>}$.}

It is not hard to show that SEP holds 
if the set $\cP$ is a singleton. 
But, in general, it is not obvious when a set of probability measures
is separable.  
We thus provide a characterization of separability, in the spirit of
Epstein and LeBreton's \citeyear{Epstein1993} rectangularity condition.
\fullv{
We actually provide two conditions, one for
the case of prior-by-prior updating, and another for the case of
likelihood updating. 
These definitions use the notion of \emph{maximum weighted expected
  value of $\theta$},  
defined as $\overline{E}_{\cP^+}(\theta) = \sup_{(\Pr,\alpha) \in \cP^+} \sum_{s \in S} {\alpha \Pr(s)}\theta(s).$
We use $\overline{X}$ to denote the closure of a set $X$.
}

\fullv{
\begin{definition}[$\chi$-Rectangularity]\label{def:rectangular}
A set $\cP^+$ of weighted probability measures is \emph{$\chi$-rectangular} ($\chi \in \{ p,l\}$) if
for all measurable sets $E$ and $F$,
\begin{enumerate}
\item[(a)] if
$(\Pr_1,\alpha_1), (\Pr_2,\alpha_2), (\Pr_3,\alpha_3) \in \cP^+$, 
$\Pr_1( E\cap F) > 0 $, and $\Pr_2( E^c \cap F) > 0$, then
$$
\alpha_3 {\Pr}_3( E \cap F) \alpha^{\chi}_{1,E\cap F} {\Pr}_1|(E \cap
F) + \alpha_3 {\Pr}_3(E^c\cap F) 
\alpha^{\chi}_{2, E^c \cap F} {\Pr}_2|(E^c\cap F) \in \overline{C(\cP^+ \mid^{\chi} F)}, 
$$
\item[(b)] for all $\delta > 0$, if $\ucP(F) > 0$, then there
  exists $(\Pr,\alpha) \in \cP^+ |^{\chi} F$  such that
$\alpha( \delta \Pr(E \cap F) + \Pr(E^c \cap F))  >
  \sup_{(\Pr',\alpha') \in \cP^+} \alpha' \Pr'(E^c \cap F)$, and  
\item[(c)] for all nonnegative real vectors $\theta \in \R^{|S|}$, 
$$\begin{array}{ll}
&\sup_{(\Pr,\alpha) \in \cP^+|^{\chi}F} \alpha \left( \Pr(E \cap F ) \overline{E}_{\cP^+|^{\chi}(E \cap F)}(\theta) + \Pr(E^c\cap F) \overline{E}_{\cP^+|^{\chi}(E^c \cap F)}(\theta) \right) \geq \overline{E}_{\cP^+|^{\chi}F}(\theta).   
\end{array}$$
\end{enumerate}
\end{definition}
Recall that Epstein and Schneider proved that rectangularity is a
condition that guarantees no preference reversal in the case of MMEU
\cite{EpsteinSchneider2003}, and Hayashi proved a similar result for
MER \cite{Hayashi2009}. 
With MMEU and MER, only unweighted probabilities are considered.
Definition~\ref{def:rectangular} essentially gives the generalization 
of Epstein and Schneider's condition to weighted probabilities.
Part (a) of $\chi$-rectangularity is analogous to the rectangularity condition of Epstein and Schneider. 
Part (b) of $\chi$-rectangularity corresponds to the assumption that
$(E\cap F)$ is non-null, which is analogous to Axiom~5 in Epstein and
Schneider's axiomatization.  
Finally, part (c) of $\chi$-rectangularity holds for MMEU when weights
are in $\{ 0,1 \}$, and thus is not necessary for Epstein and
Schneider.  
}
\commentout{
In a previous version of this paper, we considered only the case where
$F = S$, and proposed the condition that if $(\alpha,\Pr) \in \cP^+$,
then $(\alpha_{\Pr,E},\Pr|E) \in \cP^+$ and
$(\alpha_{\Pr,E^c},\Pr|E^c) \in \cP^+$, 
as a definition of richness.
It is not hard to see that when $F=S$, this older condition implies
condition (a) of our current definition of richness. 
The natural extension of the older condition to general sets $F$ also implies condition (a) of our current definition of richness.
Condition (b) is a fairly mild technical assumption, so the only
significant restriction that we have added is condition (c).}
It is not hard to show that we can replace condition (a) above by the
requirement that $\cP^+$ is closed under conditioning, 
in the sense that if $(\Pr,\alpha) \in \cP^+$, then so
are $(\Pr|(E \inter F),\alpha)$ and $(\Pr|(E^c \inter F),\alpha)$.

\fullv{
As the following result shows, $\chi$-rectangularity is indeed sufficient to give us
Axioms~\ref{axiom:DC-M}--\ref{axiom:sens} under prior-by-prior updating and likelihood updating.
}
\shortv{
As the following result shows, richness is indeed sufficient to give us Axioms~\ref{axiom:DC-M}--\ref{axiom:sens} under likelihood updating.
}
\shortv{
}
\begin{theorem}\label{thm:rich}
If $C(\cP^+)$ is closed and convex, then Axiom~\ref{axiom:DC-M} holds
for the family of choices $C_{M}^{\regret,\cP^+ |^{\chi} E}$ if and
only if $\cP^+$ is $\chi$-rectangular. 
\end{theorem}
\fullv{
The proof that $\chi$-rectangularity implies Axiom~\ref{axiom:DC-M}
requires only that $C(\cP^+)$ be closed (i.e., convexity is not
required). 
Hayashi \citeyear{Hayashi2011} proves an analogue of
Theorem~\ref{thm:rich} for MER using prior-by-prior updating.  He also
essentially assumes that the menu includes forgone opportunities, but his
interpretation of forgone opportunities is quite different from ours.
He also shows that if forgone opportunities are not included in the
menu, then the
set of probabilities representing the DM's uncertainty at all but the
initial time must be a singleton.
This implies that the DM must behave like a Bayesian at all but the
initial time, since MER acts like expected utility maximization
if the DM's uncertainty is described by a single probability measure.

Epstein and Le Breton \citeyear{Epstein1993} took this direction even further and prove that, if a few
axioms hold, then only Bayesian beliefs can be dynamically
consistent.
While Epstein and Le Breton's result was stated in a menu-free
setting, if we use a constant menu throughout the decision problem,
then our model fits into their framework.
At first glance, their impossibility result may seem to contradict our
sufficient conditions for no preference reversal.
However, Epstein and Le Breton's impossibility result does not apply
because one of their axioms,
$P4^c$, does not hold for MER (or MWER).
\shortv{ 
We discuss this in more detail in the full paper.
}
\fullv{
For ease of exposition, we give $P4^c$ for static decision problems.
Given acts $f$ and $g$ and a set $T$ of states, let $fTg$ be the act
that agrees with $f$ on $T$ and agrees with $g$ on $T^c$.  Given an
outcome $x$, let $x^*$ be the constant act that gives outcome $x$ at all states.
\begin{axiom}[Conditional weak comparative probability]\label{axiom:cwcp}
For all events $T,A,B$, with $A \cup B \subseteq T$, outcomes $w, x,
y$, and $z$, and acts $g$,
if $ w^*Tg \succ x^*T g,$  $z^*T g \succ y^*T g$,  and
$(w^*A x^*)T g \succeq (w^*B x^*)T g$, then $(z^*A y^*)T g \succeq
(z^*B y^*)T g$.
\end{axiom}

$P4^c$ implies Savage's $P4$, and does not hold for
MER and MWER in general.
%
For a simple counterexample, let
$S=\{s_1,s_2,s_3\}$,
$X = \{o_1, o_5, o_7,  o_{10}, o_{20}, o_{23}\}$,
$A=\{s_1\}$, $B=\{s_2\}$, $T=A\cup B$,
$u(o_k) = k$,
$g$ is the act such that $g(s_1) = o_{20}$, $g(s_2) = o_{23}$, and
$g(s_3) = o_5$.
Let $\cP = \{p_1, p_2, p_3\}$, where
\begin{itemize}
\item $p_1(s_1) = 0.25$ and $p_1(s_2)=0.75$;
\item $p_2(s_3) = 1$;
\item $p_3(s_1)=0.25$ and $p_3(s_3)=0.75$.
\end{itemize}
Let the menu $M = \{o_1^*, o_7^*, o_{10}^*, o_{20}^*,g\}$.
Let $\succeq$ be the preference relation determined by MER.
The regret of $o_{10}^*Tg$ is $15$ (this is the regret with respect to
$p_2$), and the
regret of $o_7^*Tg$ is $15.25$ (the regret with respect to
$p_1$), therefore
$o_{10}^*T g \succ o_7^*T g$.
It is also easy to see that the regret of $o_{20}^*Tg$ is $15$ (the
regret with respect to $p_2$),  and the regret of $o_1^*T g$ is $21.25$
(the regret with respect to $p_1$), so $o_{20}^*T g \succ o_1^*T g$.
Moreover, the regret of $(o_{10}^* A o_7^*) T g$ is $15$ (the regret
with respect to $p_2$),
 and the regret of $(o_{10}^* B o_1^*)T g$ is $15$ (the regret with
 respect to $p_2$), so $(o_{10}^*A o_7^*)T g \succeq (o_{10}^*B o_1^*)T g$.
However, the regret of $(o_{20}^*A o_1^*)T g$ is $16.5$ (the regret with
respect to $p_1$),
 and the regret of $(o_{20}^*B o_1^*)T g$ is $16$ (the regret with with
 respect to $p_3$),  therefore $(o_{20}^*A o_1^*)T g \not\succeq
 (o_{20}^*B o_1^*)T g$.
Thus, Axiom~\ref{axiom:cwcp} does not hold (taking $y=o_1,x=o_7,w=o_{10},z=o_{20}$).
}

Siniscalchi \citeyear[Proposition 1]{Siniscalchi2011} proves that his
notion of dynamically consistent conditional preference systems
must essentially have beliefs that are updated by Baysian updating.
However, his result does not apply in our case either,
because it assumes consequentialism: that the conditional
preference system treats identical subtrees equally, independent of
the greater decision tree within which the subtrees belong.
This does not happen if, for example, we take forgone opportunities
into account.

\commentout{
We show by example that if a DM's beliefs satisfy the richness
condition, then he may not be an expected-utility maximizer.

\begin{example}
It is not difficult to verify that the set of beliefs depicted in
Table~\ref{tab:ambiguityaverserich} is rich.
Consider the choices shown in Table~\ref{tab:ambiguityaverserich}.
A DM with beliefs  $\{{p}_1, {p}_2\}$ using the decision rule MWER
would chooses $A$ from
$\{A,B\}$, $A'$ from $\{A',B'\}$, and $C$ from $\{C,D\}$.
It is not hard verify that no probability distribution over $\{s_1,s_2,s_3\}$
that can result in these preferences for a decision maker using SEU.
\begin{table}[h]
	\centering
		\begin{tabular}{|c|c|c|c|}\hline
			& $s_1$ & $s_2$ & $s_3$ \\ \hline
		${p}_1$ & $\frac{2}{3}$ & 0 & $\frac{1}{3}$ \\
		${p}_2$ & 0 & $\frac{2}{3}$ & $\frac{1}{3}$ \\ \hline
		$A$ & 0 & 0 & $\epsilon$  \\
		$B$ & 1 & 0 & 0  \\ \hline
		$A'$ & 0 & 0 & $\epsilon$  \\
		$B'$ & 0 & 1 & 0  \\ \hline
		$C$ & 1 & 1 & 0  \\
		$D$ & 0 & 0 & 1  \\ \hline
		\end{tabular}
	\caption{Example showing that richness does not force preferences to be expected-utility maximizing. }
	\label{tab:ambiguityaverserich}
\end{table}
\hfill \wbox
\end{example}
}

\fullv{
\commentout{
\section{Consistent planning and unachievable plans} \label{sec:cp}
We assume that at each history, the decision maker's choice is defined
by a choice function $C_{M,E}$.
Our discussion applies to any menu-dependent decision rule, including
minimax regret, MER, and MWER.
To determine $M$, our version of consistent planning also depends on
a function $\mu$ that maps each history to a menu to be used
in that history.
Recall that in Section~\ref{sec:FO}, to determine the menu component
$M$ of the choice function
$C_{M,E}$ used at a history $h$, we used a \emph{menu-selection function} $\mu$.
The menu $\mu(h)$ is the menu relative to which choice are computed at $h$.
Here, we still require a menu-selection function.
However, unlike in the earlier sections, the menu-selection function
$\mu$ here requires more input.
This is because a sophisticated DM can potentially use a menu that
depends not only on what has been done and what is known about
nature's state, but also depends on what will be done in the future.
As a result, in addition to the history $h$, here
$\mu$ also takes as input a set of plans for each successor of $h$.
Intuitively, this set represents the consistent planning solutions at
these subsequent histories;
we may want $\mu(h)$ to exclude plans that are not
part of the consistent planning solutions at subsequent histories.
We sometimes refer to $\mu$ as a \emph{menu attitude}.
For example, $\mu$ could be a menu attitude that always return the
initial menu: i.e., $\mu(h) = \mu(M_{\< s \>})$ for all histories $h$.
We emphasize that a menu attitude is viewed as a fixed attribute of
the DM that cannot be changed; but a sophisticated DM can plan around
her own menu attitude.

We now describe how the DM computes a consistent planning solution for
a dynamic decision problem.
Recall that a dynamic decision problem is an extensive-form game $T$.
The consistent planning solution, $CP_T$, is computed by backward induction.
At the \emph{terminal histories}, or leaves in the decision tree,
there are no choices to be made.
If we move up one level to a non-terminal history $h$, then the DM can
choose from a set $A(h)$ of \emph{actions} (recall that a plan
maps a history to a successor history, which implicitly determines the
action taken).
Let $CP_T(h)$ denote the (set of) consistent planning solutions (i.e.,
optimal plans) at history $h$.
$CP_T(\< s\>)$ is the (set of) consistent planning solutions for $T$.
\begin{definition}[Menu-Dependent Consistent Planning]
Given a dynamic decision problem $T$ and a menu attitude $\mu$, we define
$CP_{T}$ inductively.
For all terminal histories $h$,
define
$$CP_{T}(\mu,h) = C_{M, E(h) }(M_h),
\mbox{ where  $M = \mu(h, M_h)$.}
$$
\noindent Inductively, for all histories $h=\langle s,a_1,\ldots,a_k
\rangle$ for which $CP_{T}(\langle s,a_1,\ldots,a_{k+1} \rangle)$ is
defined for all $a_{k+1} \in A(h)$, we define
$$CP_{T}(h) = C_{M, E(h) }(M_h \cap M), 
\mbox{ where  $M = \mu(h, \{CP_{T}(\langle s,a_1,\ldots,a_{k+1}
\rangle): a_{k+1} \in A(h)\})$}.
$$
\end{definition}
\noindent Note that in the inductive step, $\langle s,a_1,\ldots,a_k
\rangle$ is always an history.
The menu, $M = \mu(h, \{CP_{T}(\langle
s,a_1,\ldots,a_{k+1}\rangle): a_{k+1} \in A(h)\})$, depends on the
already-computed consistent planning solutions $CP_T$ at the child
nodes.
The specific menu that will be used depends on the menu attitude
function $\mu$.
In the next section, we examine several candidates for $\mu$.
}

\commentout{
In general, there are two types of plans that the DM may or may not take into consideration when computing regret:
\begin{itemize}
\item plans that have been rendered infeasible by previous steps: \emph{forgone opportunities}.
\item plans that are not forgone but nonetheless cannot be committed to: \emph{unachievable plans}.
\end{itemize}
We define unachievable plans more formally below.
By definition, the set of forgone opportunities is disjoint from the set of unachievable plans.
Moreover, once the set of forgone opportunities and the set of unachievable plans are removed, all remaining plans are feasible plans that can actually be carried out.
Therefore, the largest reasonable set of plans that can be ignored is the union of the forgone opportunities and the unachievable plans.

Although it is possible to consider arbitrary subsets of all plans in
the mother decision tree, some natural classes of menus are
particularly interesting, namely, those generated by a systematically
eliminating
an entire class of
plans.
Thus, we examine four approaches to choosing menus, depending on how we deal with forgone opportunities and unachievable plans, and we characterize the implications of using each of these approaches, under the assumption of sophistication.

We consider four {menu attitudes}, determined by whether the menu
includes or excludes forgone opportunities, and whether it
includes or excludes unachievable plans.
We denote these four {menu attitudes} by $(FO,UP)$, $(\overline{\mbox{FO}},UP)$,
$(FO,\overline{\mbox{UP}})$, and $(\overline{\mbox{FO}},\overline{\mbox{UP}})$, respectively.
These menu attitudes yield possibly different menus at each information node.
We show that each menu attitude results in a different behavioral pattern.

The climate change example depicted in Figure~\ref{fig:climateChangeExample} illustrates how the four different menu attitudes result in four different plans being chosen for the same decision problem.
Figure~\ref{fig:climateChangeExample} is a compact representation of the decision problem.
There are four actions that nature can choose from: ``low sensitivity and technology is not ready'', ``high sensitivity and technology is not ready'', ``low sensitivity and technology is ready'', and ``high sensitivity and technology is ready''.
In this case, the true state does not affect the available actions in
each stage, so we draw only one of the four subtrees, listing all
four outcomes for each terminal node.
We do not include nature's action.

\begin{figure}[h]
	\centering
		\includegraphics{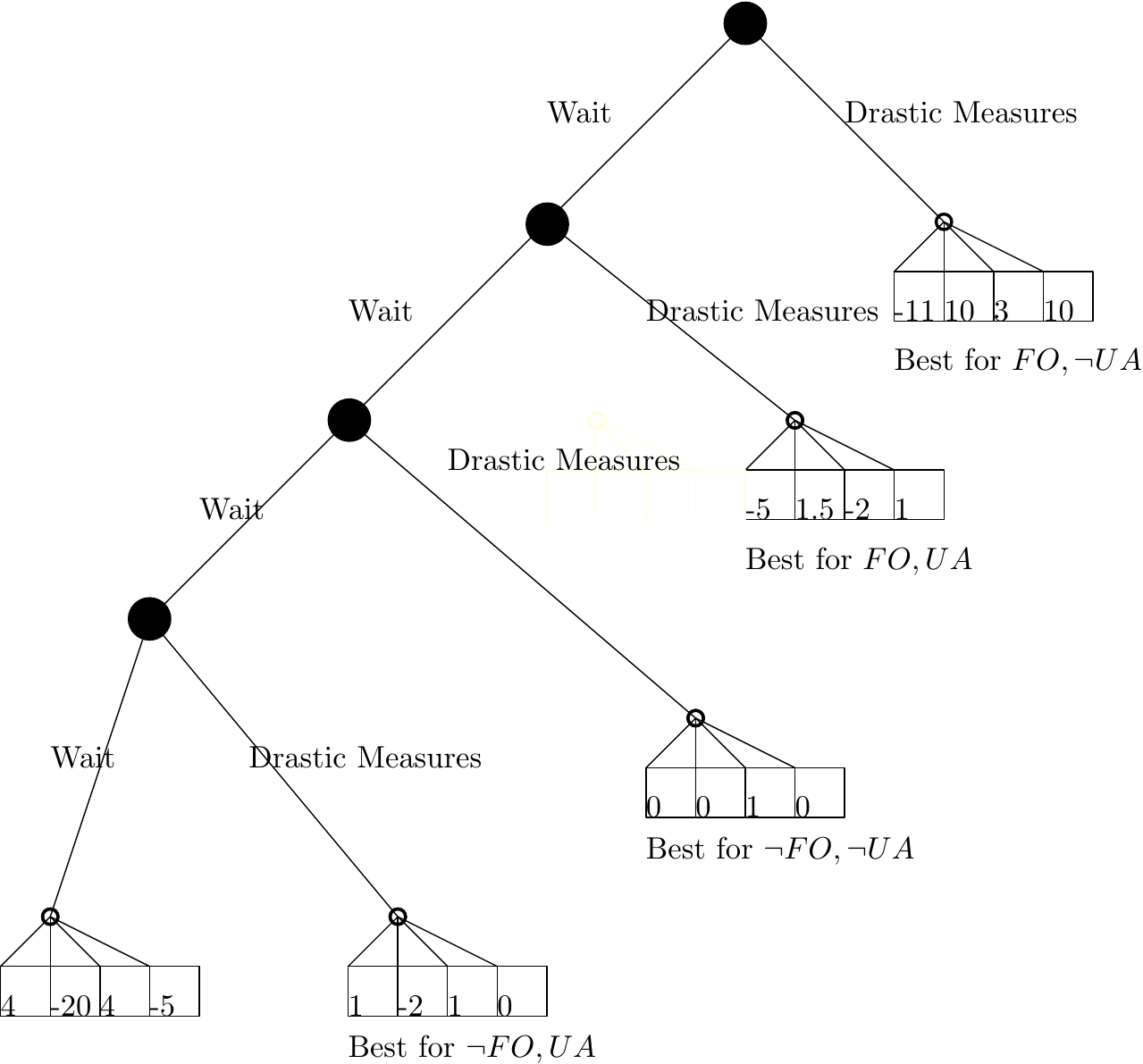}
	\caption{Climate change example illustrating four different
          choices depending on which menu is used for backward
          induction. The four possible states, from left to right, are
          $ls\overline{r}$, $hs\overline{r}$, $ls r$, and $hs r$,
          respectively. The four different menu attitudes
          result in four different choices.}
	\label{fig:climateChangeExample}
\end{figure}

The payoffs in the climate change example depend on whether the
technology for eco-friendly alternatives is ready or not, the amount
of time that the governments wait before taking action, as well as the
intrinsic sensitivity of the Earth.
A high-sensitivity Earth punishes waiting because the impact of climate change will be severe.
On the other hand, if the technology is not yet ready, the costs of implementing drastic measures will be high.
In this example, ignoring forgone opportunities makes one content to wait longer than if forgone opportunities are included in regret computations.
This is expected, since waiting reduces the number of alternatives that one can feel regret about.
} 

\commentout{
We now explore in more detail the implications of the different menu
attitudes; that is, we try to answer the question ``what does the menu
attitude tell us about a decision maker?''.


To ensure that the axioms of interest are nontrivial, we need the underlying preferences over plans to be \emph{nontrivially menu-dependent}.
That is, there should be a menu where the addition of an unchosen plan changes the choice.
\begin{definition}
A family of menu-dependent preferences $C$ is \emph{nontrivially
  menu-dependent} if there exists a decision problem with distinct plans $f_1,f_2,f_3$, and a
menu $M$ containing $f_1,f_2$, such that $C_{M \cup \{f_3\}}(M) = C_{M \cup \{f_3\}}(M\cup \{f_3\}) = f_1$
and $C_{M}(M) = f_2$.
\end{definition}

It is easy to see that when there is a nonconstant utility function
 and at least two states,
regret-minimization satisfies nontrivial menu dependency.
Suppose there are three outcomes $o_1$, $\ldots, o_5$, with $u(o_j) =
u_j$, and $u_j > u_{j-1}$, and $o_5 - o_3 > o_3 - o_2 > o_4 - o_3$, and at least two states, $s_1$ and $s_2$.
Consider three acts $f_1$, $f_2$, and $f_3$.  Suppose that for all states
$s \ne s_1, s_2$, $f_j(s) = o_1$, for $j=1,\ldots 5$.  In addition:
\begin{itemize}
\item $f_1(s_1) = o_2$ and $f_1(s_2) = o_4$;
\item $f_2(s_1) = o_3$ and $f_2(s_2) = o_3$;
\item $f_3(s_1) = o_1$ and $f_3(s_2) = o_5$.
\end{itemize}
The minimax regret decision rule would choose $f_1$ from the menu
$\{f_1,f_2,f_3\}$ since its maximum regret is $o_3 - o_2$ while $f_2$ has a maximum regret of $o_5 - o_3$.
However it chooses $f_2$ from the menu $\{f_1,f_2\}$ since $f_1$ has a maximum regret of $o_3 - o_2$ while $f_2$ has a maximum regret of only $o_4 - o_3$.


If forgone opportunities are not included in the computation of
regret, then it is easy to see that the choices at a particular point
in the decision problem will be independent of what the original
decision tree was.
This can be captured by the following axiom.
For two dynamic decision problems $T$ and $T'$,
let $T \sim_{h,h'} T'$ if $T$ restricted to $h$ and its descendents
is identical up to relabeling to $T'$ restricted to $h'$ and its descendents.
For consistent planning solutions $CP_T(h)$ and $CP_{T'}(h')$, let
$CP_T(h)\sim CP_{T'}(h')$ if the two sets of plans, restricted to $h$
and $h'$, respectively,
are identical up to relabeling.
\begin{axiom}[Context Independence]
For all dynamic decision problems $T$ and $T'$, if $T \sim_{h,h'} T'$
then $CP_T(h)
\sim CP_{T'}(h')$.
\end{axiom}
} 

There may be reasons to exclude forgone opportunities from the menu.
\emph{Consequentialism}, according to Machina
\citeyear{Machina1989}, is `snipping' the decision tree at the current
choice node, throwing the rest of the tree away, and calculating
preferences at the current choice node by applying the original
preference ordering to alternative possible continuations of the tree.
With this interpretation, consequentialism implies that forgone opportunities should be removed from the menu.

Similarly, there many be reasons to exclude unachievable plans from the menu.
Preferences computed with unachievable plans removed from the menu would be independent of these unachievable plans.
This quality might make the preferences suitable for iterated elimination of suboptimal plans as a way of finding the optimal plan.
In certain settings, it may be difficult to rank plans or find the most preferred plan among a large menu.
For instance, consider the problem of deciding on a career path.
In these settings, it may be relatively easy to identify bad plans, the elimination of which simplifies the problem.
Conversely, computational benefits may motivate a decision maker to ignore unachievable plans.
That is, a decision maker may choose to ignore unachievable plans because doing so simplifies the search for the preferred solution.

\commentout{
\begin{proposition}\label{prop:contextindep}
A family $CP$ of choice functions resulting from consistent planning using a nontrivially menu-dependent set of preferences on plans satisfies context independence if and only if the menu attitude excludes forgone opportunities.
\end{proposition}

%

We now characterize the impact of including unachievable plans in the
menu.
Recall that unachievable plans are plans that are {not forgone
  opportunities} but cannot be carried out because they are not part
of the consistent planning solution. More precisely:
\begin{definition}[Unachievable Plan]
At a terminal history $h$, there are no unachievable plans.
Inductively, for all non-terminal histories $h$, a plan $f$ is unachievable at history $h$
if $f \in M_h$ and $f \notin CP_{T}(h)$.
\end{definition}

The property we discuss is an analog of the Independence of Irrelevant
Alternatives (IIA) axiom from choice theory.
The intuition is that if a consistent planning process ignores plans that cannot be carried out, then adding such plans to a decision tree will not affect the consistent planning solution.
Note that a plan $f$, along with a state $s$, determines a complete
history in the decision problem.
For any plan $g$ that is distinct from plan $f$, there must exist at
least one state $s$ and a penterminal (i.e., its sucessor is terminal)
history $h$ where $g(h) \neq f(h)$.
TODO

For any decision problem $T$, history $h$, and a plan $f$ that is unachievable at $h$,
let $T \backslash_h \{f\}$ denote the decision problem that is the same as $T$, except that at history $h$ and any prefix of $h$, the set of feasible plans is $M_h
\backslash \{f\}$, instead of $M_h$.
Thus, $f$ does not enter into any regret computations or consistent
planning computations at any prefix of $h$.
\begin{axiom}[Independence of Unachievable Plans (IUP)]
For all dynamic decision problems $T$, for all $s\in S$, if $f$ is unachievable at $h$ then
$CP_{T}(\langle s \rangle) = CP_{T \backslash_h \{f\}}(\langle s \rangle)$.
\end{axiom}
} 

\commentout{
We conclude this section by showing that procrastination can occur even with
sophisticated decision makers,
though, in this case, the sophisticated decision makers knowingly
decide to procrastinate,
in the sense that the sophisticated decision makers know that they
will not study on the second day, and yet they choose to play on the
first day.

Consider the decision problem in Figure~\ref{fig:commitment}, which is a variation on the decision problem in Figure~\ref{fig:procrastination}.
The scenario is the same as before, with the only difference being that there is now a third state, where the exam is of intermediate difficulty.
In this third state, studying for one day gives the student an additional $45$ utils.

Assume that the student ignores forgone opportunities.
As before, after playing on the first day, the student would decide to play for another day, which has a worst-case regret of $20$, instead of studying, which has a worst-case regret of $25$.
Therefore, playing on the first day and then studying on the second
day is an unachievable plan.
%
If the student were to ignore unachievable plans in regret
computation, then studying both days has a worst-case regret of
$50$, while playing on both days has a worst-case regret of $60$.
Therefore, the student would start studying on the first day.
However, if the student takes unachievable plans into account when
computing regret, then studying both days has a worst-case regret of
$70$, while the worst-case regret of playing on both days remains at
$60$.
Therefore, the student would play on both days instead.

This example illustrates the fact that, when using regret, one can
often justify any
arbitrary decision by making up irrelevant alternatives. Therefore,
when modeling a decision problem, it is not only important to have a
good set of states and outcomes and acts, but the set of available
choices must also be chosen carefully. Of course, this problem applies
equally well to static decision problems.

\begin{figure}[h]
	\centering
		\includegraphics{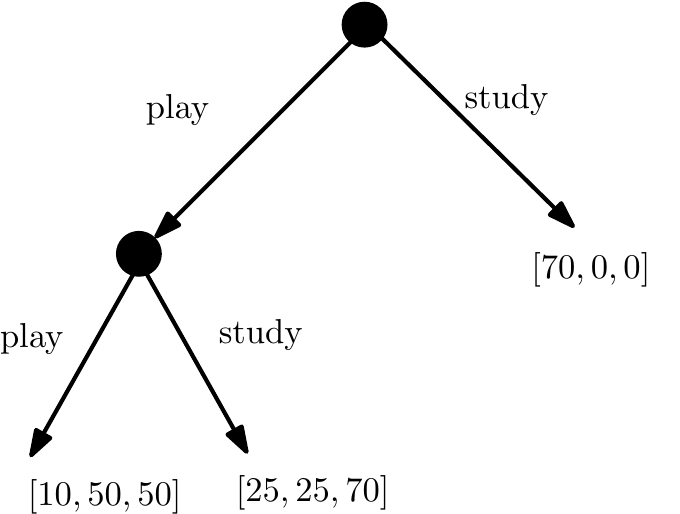}
	\caption{\label{fig:commitment}Procrastination example. In the first state, the exam is difficult; in the second state, the exam is easy.}
	\label{fig:commitmentExample}
\end{figure}
}

\commentout{
Table~\ref{tab:props} summarizes the properties of each of the menu attitudes.
Note that context independence and IUP together identify the four menu
attitudes.
In other words, these are a subset of the axioms characterizing the different choice behavior.
This is true for all forms of regret minimization, as well as
arbitrary choice functions.
\begin{table}[h]
	\centering
		\begin{tabular}{|c|c|c|}\hline
							& Context independence& IUP  \\ \hline
		$FO,UP$			 	& No	& No		\\ \hline
		$FO,\lnot UP$ 		& No	& Yes	 \\ \hline
		$\lnot FO,UP$ 		& Yes	& No		 \\ \hline
		$\lnot FO,\lnot UP$ 	& Yes	& Yes	 \\ \hline
		\end{tabular}
	\caption{Properties of the different menu attitudes.}
	\label{tab:props}
\end{table}
} 

}

\fullv{
\section{Conclusion} 
\label{sec:conclusion}
In dynamic decision problems, it is not clear which menu should be used to compute regret.
However, if we use MWER with likelihood updating, then in order to avoid preference reversals, we need to include all initially feasible plans in the menu, as well as richness conditions on the beliefs.
Another, well-studied approach to circumvent preference reversals is \emph{sophistication}.
A sophisticated agent is aware of the potential for preference
reversals, and thus uses backward induction to determine the \emph{achievable plans}, which are the plans that can
actually be carried out.
In the procrastination example, a sophisticated agent would know that
she would not study the second day.
Therefore, she knows that playing on the first day and then studying on the second day is an unachievable plan.

Siniscalchi \citeyear{Siniscalchi2011} considers a specific type of sophistication, called
\emph{consistent planning}, based on earlier definitions of Strotz \citeyear{Strotz1955} and Gul and Pesendorfer \citeyear{GP2005}.
Assuming a filtration information
structure, Siniscalchi axiomatizes behavior resulting
from consistent planning using any menu-independent decision rule.%
\footnote{Siniscalchi considers a more general
  information structure where the information that the DM receives can
  depend on her actions in an unpublished version of his paper
\cite{Siniscalchi2009}.}
With a menu-dependent decision rule, we need to consider the choice of menu when using consistent planning.
Hayashi \citeyear{Hayashi2009} axiomatizes sophistication using regret-based choices, including MER and the smooth model of anticipated regret, under the fixed filtration information setting.
However, in his models of regret, Hayashi assumes that the menu that the DM uses to compute regret includes only the achievable plans.
In other words, forgone opportunities and those plans that are not achievable are excluded from the menu.
It would be interesting to investigate the effect of including such
in the menus of a sophisticated DM.
A sophisticated decision maker who takes unachievable plans into
account when computing regret can be understood as being
``sophisticated enough'' to understand that her preferences may change
in the future, but not sophisticated enough to completely ignore the
plans that she cannot force herself to commit to when computing regret.
On the other hand, a sophisticated decision maker who ignores
unachievable plans does not feel regret for not being able to commit
to certain plans.

\commentout{
Hanany and Klibanoff provide an alternative method to consistent
planning that achieves dynamic consistency. 
For maxmin expected utility, Hanany and Klibanoff \citeyear{HK07} propose applying Bayes' rule to subsets of the set of priors, ``where the specific subset depends on the preferences, the conditioning event, and the choice problem.''
In \cite{HananyKlibanoff2009}, they expand this approach to general ambiguity averse, including regret-based preferences.
}
Finally, we have only considered ``binary'' menus in the sense that an act is either in the menu and affects regret computation, or it is not.
A possible generalization is to give different weights to the acts in the menu, and multiply the regrets computed with respect to each act by the weight of the act.
For example, with respect to forgone opportunities, ``recently forgone'' opportunities may warrant a higher weight than opportunities that have been forgone many timesteps ago.
Such treatment of forgone opportunities will definitely affect the behavior of the DM.
}
}

\appendix

\commentout{
\section{The secretary problem\label{sec:secretary}}
In this section, we consider regret-minimization in a well-known
optimal stopping problem: {the secretary problem}
\citeyear{Ferguson}.
The secretary problem shows that the menu attitude substantially affects the decision maker's behavior.

Since the secretary problem and its solutions are relatively simple,
we just explain the problem somewhat informally.
In the general monotonic utility secretary problem, there are $n$ applicants for a single open position, with one coming in at each round.
The decision maker can rank the applicants that she has seen relative to one another, but does not know how the applicants seen rank against unseen applicants.
At each round, the decision maker can either choose to hire the current applicant and end the decision process, or she can forgo the opportunity to ever hire this applicant, and move on to the next round.
The reward for hiring an applicant depends on the rank of the hired applicant among all $n$ applicants.

More formally, at each round $j$, the decision maker observes $X_j$, the relative rank of the $j$th applicant among the first $j$ applicants observed.
The reward for hiring an applicant with absolute rank $k$ is $U(k)$.
For simplicity, we assume that the reward is monotonically decreasing
in the absolute rank of the applicant.

This problem has been well studied in the context of maximizing
expected utility with respect to a uniform prior distribution over all
possible sequences of absolute rankings of the candidates \citeyear{Ferguson}.
In this case, it is known that the optimal strategy is always a \emph{threshold rule}.
That is, at any point in the decision problem, there exists a
threshold such that
the decision maker should stop and hire the applicant if the
applicant's relative ranking is better than the threshold.
The actual optimal thresholds are characterized by recursive formulas
that are quite complicated \citeyear{Ferguson}.

We now consider minimax regret behavior.
The same analysis can also be applied to MER or MWER, but there is
more overhead, so we have to then take into account $\cP$ (or $\cP^+$).
We first consider the case where forgone opportunities do not affect regret.

At the final round $n$, the DM must accept whichever applicant comes along.
At round $j < n$, the worst-case regret of stopping is realized if the
applicant in the next round is the absolute best, and the current
candidate has the lowest absolute rank possible that is consistent
with its relative rank.
If $X_j = x_j$, then the worst-case regret for stopping is:
$$R^{j}_{stop}({x_{j}}) = U(1) - U(x_j + (n-j)).$$

The worst-case regret of continuing at round $j$ is realized if the decision maker keeps seeing worse and worse applicants at all subsequent rounds.
It is not hard to see that, regardless of the stopping rule used after round $j$, it is possible to stop at the absolute worse applicant.
Therefore, the largest possible regret occurs if the round $j$ applicant turns out to have absolute rank equal to $x_j$.
Thus, the worse-case regret of continuing at round $j$ is:
$$R^{j}_{cont}(\vec{x_j}) = U(x_j) - U(n).$$

Suppose that the decision maker uses the rule of stopping if and only if the worst-case regret for stopping is less than the worst-case regret for continuing.
That is, at round $j$, the decision maker stops if and only if
$$ U(1) - U(x_j + (n-j)) < U(x_j) - U(n).$$

\noindent Like the case of maximizing expected utility, the cut-off value at each round {does not} depend on the history of applicants seen in the first $j-1$ rounds.
This is the manifestation of the context independence axiom in the secretary problem.

If the utility difference between consecutively ranking applicants is some constant $c$, then the optimal cut-off value can be simply expressed as
$$ x_j < \frac{j + 1}{2}.$$

Table~\ref{tab:cutoffs} compares the cut off values of maximizing expected utility (assuming a uniform distribution on all possible sequences of absolute rankings) \citeyear{Ferguson} and minimax regret without forgone opportunities, when $n=10$.
In this example, the cut-off value for minimax regret is always weakly larger than the cut-off value for maximizing expected utility.
This means that at every stage, a minimax-regret decision maker is
willing to stop at a worse candidate than her expected utility
maximizing counterpart.

\begin{table}[h]
	\centering
		\begin{tabular}{|c|c|c|}\hline
		Round & EU & Minimax Regret (no forgone opportunities)  \\ \hline
		1 & - & -  \\ \hline
		2 & 1 & 1  \\ \hline
		3 & 1 & 1  \\ \hline
		4 & 2 & 2  \\ \hline
		5 & 2 & 2  \\ \hline
		6 & 2 & 3  \\ \hline
		7 & 2 & 3  \\ \hline
		8 & 3 & 4  \\ \hline
		9 & 4 & 4  \\ \hline
		\end{tabular}
	\caption{\label{tab:cutoffs}Cut-off values for $n=10$. Each entry represents the relative ranking for which it is strictly better to stop.}
\end{table}

Next, we include forgone opportunities in the computation of regret.
At round $n$, our assumption that hiring even the worst applicant is better than not hiring anyone at all means that at round $n$, the regret-minimizing option is to hire the $n$th applicant.

At round $j < n$, the worst-case regret of stopping is identical to the previous case, where forgone opportunities are not part of the regret computations.
The worst-case regret is realized if the applicant in the next round is the absolute best, and the current applicant ranks lower than all future applicants.
The worst-case regret is:
$$R^{j}_{stop}(\vec{x_{j}}) = U(1) - U(x_j + (n-j)).$$
The worst-case regret of continuing is different when forgone opportunities can affect regret.
As before, regardless of what the decision maker does after continuing, it is always possible to stop at the absolute worst applicant.
The regret, however, now involves comparison against hiring any of the $j$ forgone applicants.
The worst-case regret of continuing at round $j$ is then
$$R^{j}_{cont}(
{x_{j}}) = \max_{t \leq j} U(x_t) - U(n) = U(1) - U(n).$$

As before, the optimal stopping strategy for the decision maker is to stop if and only if the worst-case regret of stopping is less than the worse-case regret of continuing, that is, if
$$U(1) - U(x_j + (n-j)) < U(1) - U(n),$$
i.e., $U(x_j + (n-j)) < U(n)$.

Comparing with the previous case where forgone opportunities are ignored, we see that this decision maker is more conservative -- she has a {greater tendency to stop} at every round.
In fact, the decision maker must stop if the current round applicant is not relatively-worst amongst the applicants already seen.
This is consistent with the intuition that when forgone opportunities are included in the menu, then the more opportunities are passed up, the higher the potential regret; as a result a decision maker that includes forgone opportunities in the menu are less inclined to create forgone opportunities.

Unlike the case of maximizing expected utility and minimax regret without forgone opportunities, here the cut-off value at any particular stage $j$ may depend on the relative ranking of the first $j$ applicants.

As we see, there are multiple reasonable solutions to the classical secretary problem, depending on whether one is an expected-utility maximizer, a regret-minimizer who does not regret forgone opportunities, or a regret-minimizer who do regret forgone opportunities.
These decision makers are increasingly conservative.

} 

\section{Proof of Theorem~\ref{prop:norev}}\label{apdx:norev}
\newenvironment{RETHM}[2]{\trivlist \item[\hskip 10pt\hskip\labelsep{\sc #1\hskip 5pt\relax\ref{#2}.}]\it}{\endtrivlist}
\newcommand{\rethm}[1]{\begin{RETHM}{Theorem}{#1}}
\newcommand{\erethm}{\end{RETHM}}
We restate the theorem (and elsewhere in the appendix)
for the reader's convenience.
\rethm{prop:norev}
For a dynamic decision problem $D$, if
and $\mu(h)=M$ for some fixed menu $M$,
then there will be no preference reversals in $D$.
\erethm

\begin{proof}
Before proving the result, we need some definitions.
Say that an information set $I$ \emph{refines} an information set $I'$
if, for all $h \in I$, some prefix $h'$ of $h$ is in $I'$.
Suppose that there is a history $h$ such that $f, g \in M_h$ and $I(h)
= I$.  Let $fIg$ denote
the plan that agrees with $f$ at all histories $h'$ such that $I(h')$
refines $I$ and agrees with $g$ otherwise.
As we now show, $fIg$ gives the
same outcome as $f$ on states in
$E = E(h)$ and the same outcome as $g$ on states in $E^c$;
moreover, $fIg \in M_h$.

Suppose that $s(h) = s$ and that $s \in E$.  Since $E(h) = E$,
there exists a history $h' \in I(h)$ such that $s(h') = s'$ and
$R(h') = R(h)$.
Since $f, g \in M_h$, there must exist some $k$ such
that $f^k(\<s\>) = g^k(\<s\>) = h$
(where, as usual, $f^0(\<s\>) = \<s\>$
and for $k' \ge 1$, $f^{k'}(\<s\>) = f(f^{k'-1}(\<s\>))$).
  We claim that for all $k' \le k$,
$f^{k'}(\<s'\>) = g^{k'}(\<s'\>)$, and $f^{k'}(\<s'\>)$ is in the same
information set as $f^{k'}(\<s\>)$.  The proof is by induction on
$k'$.  If $k' = 0$, the result follows from the observation that since
$\<s\>$ is a prefix of $h$, there must be some prefix of $h'$ in
$I(\<s\>)$.  For the inductive step, suppose that $k' \ge 1$.  We must
have $f^{k'}(\<s\>) =
g^{k'}(\<s\>)$ (otherwise $g$ would not be in $M_h$).  Since
$g^{k'-1}(\<s\>) = f^{k'-1}(\<s\>)$ and $f^{k'-1}(\<s'\>) =
g^{k'-1}(\<s'\>)$ are in the same information set, by the inductive
hypothesis, $g$ must perform the same action at $g^{k'-1}(\<s\>)$
and $g^{k'-1}(\<s'\>)$, and must perform the same action at $f^{k'-1}(\<s\>)$
and $f^{k'-1}(\<s'\>)$.  Since $g^{k'}(\<s\>)$ and
$f^{k'}(\<s\>)$ are both prefixes of $h$, $g$ and $f$ perform the
same action at $f^{k'-1}(\<s\>) = g^{k'-1}(\<s\>)$.  It follows that
$f$ and $g$ perform the same action at $f^{k'-1}(\<s'\>) =
g^{k'-1}(\<s'\>)$, and so $f^{k'}(\<s'\>) = g^{k'}(\<s'\>)$.  Thus,
$g^{k'}(\<s'\>)$ must be a prefix of $h'$, and so must be in the same
information set as $f^{k'}(\<s\>)$.   This completes the inductive
proof.

Since $f^k(\<s'\>) = g^k(\<s'\>) =h'$, it follows that $f^k(\<s'\>) =
(fIg)^k(\<s'\>)$.  Below $I$, all the information sets are refinements
of $I$, so by definition, for $k' \le k$, we must $f^{k'}(\<s'\>) =
(fIg)^{k'}(\<s'\>)$.  Thus, $f$ and $fIg$ give the same outcome for
$s'$, and hence all states in $E$.  Note it follows that
$(fIg)^k(\<s\>) = h$, so $fIg \in M_h$.

For $s' \notin E$ and all $k'$, it cannot be the case that
$I((fIg)^{k'}(\<s'\>))$ is a refinement of $I$, since
the first state in $(fIg)^{k'}(\<s'\>))$ is $s'$, and no history in a
refinement of $I$ has a first state of $s'$.  Thus, $fIg^{k'}(\<s'\>) =
g^{k'}(\<s'\>)$ for all $k'$, so $f$ and $fIg$ give the same outcome for
$s'$, and hence all states in $E^c$.


Returning to the proof of the proposition, suppose that $f \in C_{\mu,h}(M_h)$,
$h'$ is a history extending $h$, and $f \in M_{h'}$.
We want to show that $f \in C_{\mu,h'}(M_{h'})$.
By perfect recall, $E(h') \subseteq E(h)$.
Suppose, by way of contradiction, that $f \notin C_{\mu,h'}(M_{h'})$.
Since $f \in C_{\mu,h'}(M_{h'})$, we cannot have $E(h') = E(h)$, so
$E(h') \subset E(h)$.
Choose $f' \in C_{\mu,E(h')}(M_{h'})$ and $g \in C_{\mu,E(h')^c \cap E(h)}(M_{h'})$
(note that $C_{\mu,E(h')}(M_{h'}) \ne \emptyset$ and
$C_{\mu,E(h')^c \cap E(h)}(M_{h'}) \ne \emptyset$ by Axiom 3).
Since $f', g \in M_{h'}$ (by Axiom 3), $f'I(h')g$ is in $M_{h'}$.
Since $f'I(h')g$ and $f'$, when viewed as acts, agree on states in
$E(h')$, we must have $f'I(h')g \in C_{\mu,E(h')}(M_{h'})$ by
Axiom~\ref{axiom:offE}.
Similarly, since $f'I(h')g$ and $g$, when viewed as acts, agree on states in
$E(h')^c \cap E(h) $, we must have $f'I(h')g \in C_{\mu,E(h')^c \cap E(h)}(M_{h'})$.
Therefore, by Axiom~\ref{axiom:DC-M}, $f'I(h')g \in C_{\mu,h}(M_{h'})$.
Also by Axiom~\ref{axiom:DC-M}, since $f \notin C_{\mu,h'}(M_{h'})$,
we must have $f \notin C_{\mu,h}(M_{h'})$.
By Axiom~\ref{axiom:sens}, this implies that $f \notin
C_{\mu,h}(M_{h})$ (since $M_{h'} \subseteq M_h$),
giving us the desired contradiction.
\end{proof}

\section{Proof of Theorem~\ref{thm:equivalence}}\label{apdx:thm1}

\rethm{thm:equivalence}
If $\cP^+$ is a set of weighted distributions
on $(S,\Sigma)$ such that $C(\cP^+)$ is closed, then the following are
equivalent: 
\begin{enumerate}
\item[(a)] For all decision problems $D$ based on $(S,\Sigma)$ and all menus
  $M$ in $D$, Axioms~\ref{axiom:DC-M}--\ref{axiom:sens} hold for
choice functions represented by $\cP^+|^l E$ (resp., $\cP^+|^p E$).
\item[(b)] For all decision problems $D$ based on $(S,\Sigma)$, 
states $s \in S$, and acts $f\in M_{\< s \>}$, the
weighted regret of $f$ with respect to $M_{\< s \>}$ and $\cP^+$ is separable.
\end{enumerate}
\erethm

We actually prove the following stronger result.
\begin{theorem}
If $\cP^+$ is a set of weighted distributions on $(S,\Sigma)$ such that $C(\cP^+)$ is closed, then the following are equivalent:
\begin{enumerate}
\item[(a)] For all decision problems $D$ based on $(S,\Sigma)$, Axioms~\ref{axiom:DC-M}--\ref{axiom:sens} hold for menus of the form $M_{\< s \>}$ for 
choice functions represented by $\cP^+|^l E$ (resp., $\cP^+|^p E$).
\item[(b)] For all decision problems $D$ based on $(S,\Sigma)$ and all menus
  $M$ in $D$, Axioms~\ref{axiom:DC-M}--\ref{axiom:sens} hold for choice functions represented by $\cP^+|^l E$ (resp., $\cP^+|^p E$).
\item[(c)] For all decision problems $D$ based on $(S,\Sigma)$, 
states $s \in S$, and acts $f\in M_{\< s \>}$, the weighted regret of $f$ with respect to $M_{\< s \>}$ and $\cP^+$ is separable.
\item[(d)] For all decision problems $D$ based on $(S,\Sigma)$, menus $M$ in $D$, and acts $f\in M_{}$, the weighted regret of
$f$ with respect to $M$ and $\cP^+$ is separable.
\end{enumerate}
\end{theorem}

\begin{proof}
Fix an arbitrary state space $S$, measurable events $E, F \subseteq S$, and a set $\cP^+$ of weighted distributions on $(S,\Sigma)$.
The fact that (b) implies (a) and (d) implies (c)
follows immediately. 
Therefore, it remains to show that (a) implies (d) and that (c) implies (b).

Since the proof is identical for prior-by-prior updating ($|^p$) and for likelihood updating ($|^l$), we use $|$ to denote the updating operator.
That is, the proof can be read with $|$ denoting $|^p$, or with $|$ denoting $|^l$.

To show that (a) implies (d), 
we first show that Axiom~\ref{axiom:DC-M} implies that for all
decision problems $D$ based on $(S,\Sigma)$, menu $M_{}$ in $D$, sets $\cP^+$
of 
weighted probabilities, and acts $f \in M_{}$,
\begin{align}
\label{equ:dcCondition1}
\regret_{M}^{\cP^+|F} (f) \geq \sup_{(\Pr,\alpha) \in \cP^+} \alpha \left( \Pr(E\cap F) \regret_{M}^{\cP^+|(E\cap F)}(f) + \Pr(E^c \cap F) \regret_{M}^{\cP^+|(E^c \cap F)}(f)\right).
\end{align}
%
Suppose, by way of contradiction, that (\ref{equ:dcCondition1}) does not hold.
Then for some decision problem $D$ based on $(S,\Sigma)$, measurable events $E,F
\subseteq S$, menu $M$ in $D$, and act $f\in M$, we have that
$$ \regret_{M}^{\cP^+|F} (f) < \sup_{(\Pr,\alpha) \in \cP^+} \alpha \left(
\Pr(E\cap F) \regret_{M}^{\cP^+|(E\cap F)}(f) + \Pr(E^c \cap F)
\regret_{M,F}^{\cP^+|(E^c \cap F)}(f)\right).$$
We define a new decision problem $D'$ based on $(S,\Sigma)$.
The idea is that in $D'$, we will have a plan $a_{f'}$ such that $a_{f'} \in
C^{\regret,\cP^+}_{{M',E\cap F}}(M'')$ and $a_{f'}\in
C^{\regret,\cP^+}_{{M',E^c \cap F}}(M'')$ and $a_{f'}\notin
C^{\regret,\cP^+}_{{M',F}}(M'')$ for some $M''\subseteq M'$, where $M'$ is the menu at the initial decision node for the DM.

We construct $D'$ as follows.
$D'$ is a depth-two tree; that is, nature makes a single move, and
then the DM
makes a single move. 
At the first step, nature choose a state $s \in F$.
At the second step, the DM chooses from the set $\{ a_g : g \in M \} \cup \{a_{f'}\}$ of actions.
With a slight abuse of notation, we let $a_g$ also denote the plan in $T'$ that chooses the action $a_g$ at the initial history $\<s\>$.
Therefore, the initial menu in decision problem $D'$ is $M'=\{ a_g : g \in M \} \cup \{a_{f'}\}$.

The utilities for the actions/plans in $D'$ are defined as follows.
For actions $\{a_g : g\in M\}$, the utility of $a_g$ in state $s$ is
just the utility of the outcome resulting from applying plan $g$ in state $s$
in decision problem $D$. 
The action $a_{f'}$ has utilities
\begin{align*}
u(a_{f'}(s)) = \begin{cases}
\sup_{g\in M} u(g(s)) - \regret_{M}^{\cP^+|(E \cap F)}(f) \text{ if } s\in E\cap F \\
\sup_{g\in M} u(g(s)) - \regret_{M}^{\cP^+|(E^c \cap F)}(f) \text{ if } s\in E^c \cap F.
\end{cases}
\end{align*}

For all states $s\in F$, we have that $u(a_{f'}(s)) \leq \sup_{g\in M} u(g(s))$.
As a result, for all states $s\in F$, we have that
\begin{align*}
\sup_{g\in M} u(g(s)) = \sup_{a_g \in M'} u(a_g(s)).
\end{align*}
Since the regret of a plan in state $s$ depends only on its payoff in
$s$ and the best payoff in $s$, it is not hard to see that the regrets
of $a_g$ with respect to $M'$ is the same as the regret of $g$ with
respect to $M$. 
More precisely, for all $g \in M$,
\begin{align*}
\regret_{M'}^{\cP^+|(E\cap F)}(a_g) & = \regret_{M}^{\cP^+|(E\cap F)}(g), \\
\regret_{M'}^{\cP^+|(E^c \cap F)}(a_g) & = \regret_{M}^{\cP^+|(E^c \cap F)}(g), \text{ and} \\
\regret_{M'}^{\cP^+|F}(a_g) & = \regret_{M}^{\cP^+|F}(g).
\end{align*}
%
By definition of $a_{f'}$, for each state $s\in E\cap F$, we have
$\regret_{M'}(a_{f'},s) = \regret_{M}^{\cP^+|(E\cap F)}(f)$, and for each state
$s\in E^c \cap F$, we have $\regret_{M'}(a_{f'},s) = \regret_{M}^{\cP^+|(E^c
  \cap F)}(f)$. 
Thus, for all $\Pr \in \cP$, if $\Pr(E\cap F) \ne 0$, then
$\regret_M^{\Pr|(E\cap F)}(f) = 
\regret_{M}^{\cP^+|(E\cap F)}(f)$, and
if $\Pr(E^c\cap F) \ne 0$, then
$\regret_M^{\Pr|(E^c\cap F)}(f) = \regret_{M}^{\cP^+|(E^c\cap
  F)}(f)$. 
If for all $(\Pr,\alpha) \in \cP^+|(E\cap F)$, $\alpha \Pr( E\cap F) = 0$,
then $\regret_{M'}^{\cP^+|(E\cap F)}(a_{f'}) = \regret_{M'}^{\cP^+|(E\cap F)}(a_f) = 0$. 
Otherwise, since there is some measure in $\cP^+|(E\cap F)$ that has weight $1$,
we must have $\regret_{M'}^{\cP^+|(E\cap F)}(a_{f'}) =
\regret_{M'}^{\cP^+|(E\cap F)}(a_f)$. 
Similarly,  $\regret_{M'}^{\cP^+|(E^c
  \cap F)}(a_{f'}) = \regret_{M'}^{\cP^+|(E^c \cap F)}(a_f)$. 
Thus,
$$\begin{array}{lll}
\regret_{M'}^{\cP^+|F}(a_{f'})
& = \sup_{(\Pr,\alpha) \in\cP^+} \alpha \left( \Pr(E\cap F) \regret_{M}^{\cP^+|(E\cap F)}(f) + \Pr(E^c \cap F)\regret_{M}^{\cP^+|(E^c \cap F)}(f) \right)\\
& > \regret_M^{\cP^+|F}(f)  \ \ \ \ \ \ \mbox{[by assumption]}\\
& = \regret_{M'}^{\cP^+|F}(a_f) \ \ \ \ \mbox{[by construction].}
\end{array}
$$
Therefore, we have $a_{f'} \in C^{\regret,\cP^+}_{{M',E\cap
    F}}(\{a_{f'}, a_f\})$, $a_{f'} \in C^{\regret,\cP^+}_{{M',E^c
    \cap F}}(\{a_{f'}, a_f\})$, and $a_{f'}\notin
C^{\regret,\cP^+}_{{M',F}}(\{a_{f'}, a_f\})$, violating
Axiom~\ref{axiom:DC-M}. 

By an analogous argument, we show that the opposite weak inequality,
\begin{align}\label{equ:otherdir}
\regret_M^{\cP^+|F} (f) \leq \sup_{(\Pr,\alpha) \in \cP^+} \alpha
\left( \Pr(E\cap F) \regret_{M}^{\cP^+|(E\cap F)}(f) + \Pr(E^c \cap
F) \regret_{M}^{\cP^+|(E^c \cap F)}(f)\right),
\end{align} is also implied by Axiom~\ref{axiom:DC-M}.
Suppose, by way of contradiction, that (\ref{equ:otherdir}) does not hold.
Then for some decision problem $D$ based on $(S,\Sigma)$, measurable events $E,F
\subseteq S$, menu $M$ in $D$, and act $f\in M$, we have that
$$ \regret_{M}^{\cP^+|F} (f) > \sup_{(\Pr,\alpha) \in \cP^+} \alpha \left(
\Pr(E\cap F) \regret_{M}^{\cP^+|(E\cap F)}(f) + \Pr(E^c \cap F)
\regret_{M,F}^{\cP^+|(E^c \cap F)}(f)\right).$$
We define a decision problem $D'$ based on $(S,\Sigma)$ 
just as in the previous case.  Specifically, we have that
$\regret_{M'}^{\cP^+|(E\cap F)}(a_{f'}) =
\regret_{M'}^{\cP^+|(E\cap F)}(a_f)$, and that 
$\regret_{M'}^{\cP^+|(E^c
  \cap F)}(a_{f'}) = \regret_{M'}^{\cP^+|(E^c \cap F)}(a_f)$. 
The one difference from the previous case is that we now have
$$\begin{array}{lll}
\regret_{M'}^{\cP^+|F}(a_{f'})
& = \sup_{(\Pr,\alpha) \in\cP^+} \alpha \left( \Pr(E\cap F) \regret_{M}^{\cP^+|(E\cap F)}(f) + \Pr(E^c \cap F)\regret_{M}^{\cP^+|(E^c \cap F)}(f) \right)\\
& < \regret_M^{\cP^+|F}(f)  \ \ \ \ \ \ \mbox{[by assumption]}\\
& = \regret_{M'}^{\cP^+|F}(a_f) \ \ \ \ \mbox{[by construction].}
\end{array}
$$
Therefore, we have $a_{f} \in C^{\regret,\cP^+}_{{M',E\cap
    F}}(\{a_{f'}, a_f\})$, $a_{f} \in C^{\regret,\cP^+}_{{M',E^c
    \cap F}}(\{a_{f'}, a_f\})$, and $a_{f}\notin
C^{\regret,\cP^+}_{{M',F}}(\{a_{f'}, a_f\})$, violating
Axiom~\ref{axiom:DC-M}. 

To complete the proof that (a) implies (d), we show that 
Axiom~\ref{axiom:DC-M} also implies that
for all decision problems $D$ based on $(S,\Sigma)$, menus $M$ in $D$,
sets $\cP^+$ of 
weighted probabilities, and acts $f \in M$, if
$\regret_{M}^{\cP^+|(E\cap F)}(f) > 0 $, then 
\begin{align}\label{equ:seppart2}
\regret_{M}^{\cP^+|F} (f) > \sup_{(\Pr,\alpha) \in \cP^+} \alpha \Pr(E^c \cap F) \regret_{M}^{\cP^+|(E^c \cap F)}(f).
\end{align}
Suppose, by way of contradiction, that (\ref{equ:seppart2}) does not hold.
Then for some decision problem $D$ based on $(S,\Sigma)$, events
$E,F\subseteq S$, menu $M$ in $D$, and act $f\in M$ such that 
$\regret_{M}^{\cP^+|(E\cap F)}(f) > 0 $ and
$$\regret_{M}^{\cP^+|F} (f)\leq \sup_{(\Pr,\alpha) \in \cP^+} \alpha \Pr(E^c \cap F) \regret_{M}^{\cP^+|(E^c \cap F)}(f).$$
We now define a new decision problem $D'$ based on $(S,\Sigma)$.
The idea is that in $D'$, we have a plan $a_{f}$ such that $a_{f}
\notin C^{\regret,\cP^+}_{{M,E\cap F}}(M')$ but $a_{f} \in
C^{\regret,\cP^+}_{{M,F}}(M')$ for some $M'\subseteq M$. 

Construct $D'$ exactly as before.
That is, in the first step, nature chooses a state $s\in S$, and in the second step, the DM chooses from the set of actions/plans $M' = \{ a_g : g \in M \} \cup \{ a_{g'}\}$.
%
For each $g\in M$, define the actions $a_g$ as before. 
We define a new action $a_{g'}$ with utilities
\begin{align*}
u(a_{g'}(s)) = \begin{cases}
\sup_{g\in M} u(g(s)), \text{ if } s\in E\cap F \\
\sup_{g\in M} u(g(s)) - \regret_{M}^{\cP^+|(E^c \cap F)}(f), \text{ if } s\in E^c \cap F.
\end{cases}
\end{align*}

It is almost immediate from the definition of $a_{g'}$ that we have
\begin{align*}
\regret_{M'}^{\cP^+|F}(a_{g'})
& = \sup_{(\Pr,\alpha) \in\cP^+} \alpha \left(\Pr(E^c \cap
F)\regret_{M}^{\cP^+|(E^c \cap F)}(f) \right) \geq
\regret_{M'}^{\cP^+|F}(a_f).
\end{align*}
However, we also have
\begin{align*}
\regret_{M'}^{\cP^+|(E\cap F)}(a_g) = 0 < \regret_{M'}^{\cP^+|(E\cap F)}(a_f).
\end{align*}
Therefore, we have $a_{f} \notin C^{\regret,\cP^+}_{{M',E\cap F}}(\{a_{g'}, a_f\})$ but $a_{f}\in C^{\regret,\cP^+}_{{M',F}}(\{a_{g'}, a_f\})$, violating Axiom~\ref{axiom:DC-M}.

%
%
We next show that (c) implies (b).
Specifically, we show that SEP for the initial menus of all decision problems
$D$ is sufficient to guarantee that 
Axioms~\ref{axiom:DC-M}--\ref{axiom:sens}
hold for menu $M$ and all choice sets $M' \subseteq M$.
It is easy to check that Axioms~\ref{axiom:offE}--\ref{axiom:sens} hold
for MWER, so we need to check only Axiom~\ref{axiom:DC-M}.

Consider an arbitrary decision problem $D$, menu $M$ in $D$,
$M'\subseteq M$, and a plan $f$ in $M'$. 
We construct a new decision problem $D'$ such that the
initial menu of $D'$ is ``equivalent'' to $M$.
%
Just as before, let $D'$ be a two-stage decision problem where in the first stage, nature chooses $s\in S$, and in the second stage, the DM chooses from the set $M_0 = \{ a_g : g \in M\}$,
where $a_g$ is defined as before.
Again, we associate each action $a_g$ with the plan that chooses
$a_g$ in $D'$. 
$M_0$ is then ``equivalent'' to $M$ in the sense that
\begin{align*}
\regret_{M_0}^{\cP^+|(E\cap F)}(a_g) & = \regret_{M}^{\cP^+|(E\cap F)}(g), \\
\regret_{M_0}^{\cP^+|(E^c \cap F)}(a_g) & = \regret_{M}^{\cP^+|(E^c \cap F)}(g), \text{ and} \\
\regret_{M_0}^{\cP^+|F}(a_g) & = \regret_{M}^{\cP^+|F}(g).
\end{align*}
Suppose that $f \in C^{\regret,\cP^+}_{M,E\cap F}(M')$ and $f \in C^{\regret,\cP^+}_{M,E^c \cap F}(M')$.
This means that for all $g \in M'$, we have
$\regret_{M_0}^{\cP^+|(E\cap F)}(a_f) \leq \regret_{M_0}^{\cP^+|(E\cap F)}(a_g)$ and $\regret_{M_0}^{\cP^+|(E^c \cap F)}(a_f) \leq \regret_{M_0}^{\cP^+|(E^c \cap F)}(a_g)$.
Therefore, we have
\begin{align*}
\regret_{M}^{\cP^+|F}(f) & =
\regret_{M_0}^{\cP^+|F}(a_f) \\
& = \sup_{(\Pr,\alpha) \in \cP^+} \alpha \left( \Pr(E\cap F) \regret_{M_0}^{\cP^+|(E\cap F)}(a_f) +  \Pr(E^c \cap F) \regret_{M_0}^{\cP^+|(E^c \cap F)}(a_f) \right) \\
& \leq \sup_{(\Pr,\alpha) \in \cP^+} \alpha \left( \Pr(E\cap F)\regret_{M_0}^{\cP^+|(E \cap F)}(a_g) +  \Pr(E^c \cap F) \regret_{M_0}^{\cP^+|(E^c \cap F)}(a_g) \right) \\
& = \regret_M^{\cP^+|F}(g),
\end{align*}
which means that $f \in C^{\regret,\cP^+}_{M,F}(M')$, as required.

Next, consider an act $g \in M'$ such that $g \notin
C^{\regret,\cP^+}_{M,E\cap F}(M')$. 
This means that
$\regret_{M_0}^{\cP^+|(E \cap F)}(a_f) < \regret_{M_0}^{\cP^+|(E \cap F)}(a_g)$ and $\regret_{M_0}^{\cP^+|(E^c \cap F)}(a_f) \leq \regret_{M_0}^{\cP^+|(E^c \cap F)}(a_g)$.
Let $(\alpha_{\Pr^*},\Pr^*) \in C(\cP^+)$ be such that
$$\begin{array}{ll} 
& \alpha_{\Pr^*} (\Pr^*(E\cap F)\regret_{M_0}^{\cP^+|(E \cap
  F)}(a_g) +  \Pr^*(E^c \cap F) \regret_{M_0}^{\cP^+|(E^c \cap F)}(a_g)\\ 
 = &\sup_{(\Pr,\alpha) \in
  \cP^+} \alpha \left( \Pr(E\cap F)\regret_{M_0}^{\cP^+|(E \cap
  F)}(a_g) +  \Pr(E^c \cap F) \regret_{M_0}^{\cP^+|(E^c \cap F)}(a_g)
\right).\end{array}$$ 
Such a pair $(\alpha_{\Pr^*},\Pr^*)$ exists, since we have assumed that
$C(\cP^+)$ is closed. 
If $\alpha_{\Pr^*}{\Pr}^*(E\cap F) = 0$, then
$\regret_{M_0}^{\cP^+|F}(a_g) = \sup_{(\Pr,\alpha) \in \cP^+} \alpha \left(
\Pr(E^c \cap F) \regret_{M_0}^{\cP^+|(E^c \cap F)}(a_g)
\right)$. 
By separability, it must be the case that $\regret_{M_0}^{\cP^+|(E \cap F)}(a_g) = 0$, contradicting our assumption that $0 \leq \regret_{M_0}^{\cP^+|(E \cap F)}(a_f) < \regret_{M_0}^{\cP^+|(E \cap F)}(a_g)$.
Therefore, it must be that $\alpha_{\Pr^*}\Pr^*(E\cap F) > 0$, and
\begin{align*}
\regret_M^{\cP^+|F}(f) & =
\regret_{M_0}^{\cP^+|F}(a_f) \\
& = \sup_{(\Pr,\alpha) \in \cP^+} \alpha \left( \Pr(E\cap F) \regret_{M_0}^{\cP^+|(E \cap F)}(a_f) +  \Pr(E^c \cap F) \regret_{M_0}^{\cP^+|(E^c \cap F)}(a_f) \right) \\
& < \sup_{(\Pr,\alpha) \in \cP^+} \alpha \left( \Pr(E\cap F)\regret_{M_0}^{\cP^+|(E \cap F)}(a_g) +  \Pr(E^c \cap F) \regret_{M_0}^{\cP^+|(E^c \cap F)}(a_g) \right) \\
& = \regret_M^{\cP^+|F}(g),
\end{align*} which means that $g \notin C^{\regret,\cP^+}_{M,F}(M')$.
\end{proof}

\section{Proof of Theorem~\ref{thm:rich} }

To prove Theorem~\ref{thm:rich}, we need the following lemma.
\begin{lemma}\label{lemma:C}
For all utility functions $u$, sets $\cP^+$ of weighted probabilities, acts $f$, and menus $M$ containing
$f$, $\regret_{M}^{\cP^+}(f) = \regret_{M}^{C(\cP^+)}(f)$.
\end{lemma}
\begin{proof}Simply observe that
\begin{align*}
\regret_M^{\cP^+}(f) &= \sup_{(\Pr,\alpha) \in \cP^+}\left( \alpha \sum_{s\in S}\Pr(s)\regret_M(f,s)  \right) \\
& = \sup_{(\Pr,\alpha) \in \cP^+}\left(  \sum_{s\in
  S}\alpha \Pr(s)\regret_M(f,s)  \right)\\
&= \sup_{\{ p :\, p \leq \alpha \Pr, (\Pr,\alpha) \in \cP^+\}}\left(
\sum_{s\in 
  S} p(s) \regret_M(f,s)  \right)\\
& = \regret_{M}^{C(\cP^+)}(f),
\end{align*}by definition.
\end{proof}

The next lemma uses an argument almost identical to one used in Lemma
7 of \cite{HalpernLeung2012}.  
\begin{lemma}\label{lem:hyperplane}
If $C(\cP^+|^{\chi} F)$ is convex and $q$ is a subprobability on $F$
not in $\overline{C(\cP^+|^{\chi}F)}$, then there exists a non-negative
vector $\theta$ such that for all $(\Pr,\alpha) \in \cP^+|^{\chi}F$,
we have 
$$\sum_{s \in F} \alpha \Pr(s) \theta(s) < \sum_{s\in F} q(s) \theta(s).$$
\end{lemma}
\begin{proof}
Given a set $\cP^+$ of  weighted probabilities, let $C'(\cP^+) = \{ p : p  \in 
\R^{|S|} \text{ and } p \leq \alpha \Pr \text{ for some } (\Pr,\alpha)
\in \cP^+ \}$. 
Note that an element $q \in C'(\cP^+)$ may not be a subprobability
measure, since we do not require that $q(s) \ge 0$.
Since $\overline{C'(\cP^+|^{\chi} F)}$ and $\{{q}\}$ are closed, convex, and
disjoint, and $\{{q}\}$ is compact, the separating hyperplane
theorem \cite{Rockafellar} says that there exist $\theta \in
\R^{|S|}$ and $c\in \R$ such 
that  
\begin{align}\label{equ:separating}
\theta \cdot {p}< c \text{ for all } {p}\in \overline{C'(\cP^+|^{\chi} F)} \text{, and } \theta \cdot {q} > c.
\end{align}
Since $\{ \alpha\Pr : (\Pr,\alpha) \in \cP^+|^{\chi} F \} \subseteq
\overline{C'(\cP^+|^{\chi} F)}$, we have that for all $(\Pr,\alpha) \in
\cP^+|^{\chi}F$, 
$$\sum_{s \in F} \alpha \Pr(s) \theta(s) < \sum_{s\in F} q(s) \theta(s).$$
Now we argue that it must be the case that $\theta(s) \geq 0$ for all
$s\in F$. 
Suppose that
$\theta(s') < 0 $ for some $s'\in F$. 
Define $p^*$ by setting
\begin{align*}
{p^*}(s) = 
\begin{cases}
0 \text{, if } s\neq s' \\
\frac{ -|c| }{|\theta(s')|} \text{, if } s = s'.
\end{cases}
\end{align*}
Note that $p^{*} \leq \vec{0}$, since for all
$s\in S$, $p^{*}(s) \leq 0$. 
Therefore, $p^{*} \in C'(\cP^+|^\chi F)$.

Our definition of $p^{*}$ also ensures that $\theta \cdot {p^{*}} = \sum_{s\in S} p^{*}(s)  \theta(s) = p^{*}(s') \theta(s') =  |c| \geq  c $.
This contradicts (\ref{equ:separating}), which says that $\theta \cdot {p} < c \text{ for all } {p}\in C'(\cP^+|^{\chi} F)$.
Thus it must be the case that  $\theta(s) \geq 0$ for all $s\in S$.
\end{proof}

We are now ready to prove Theorem~\ref{thm:rich}, which we restate here.
\rethm{thm:rich}
If $C(\cP^+)$ is closed and convex, then Axiom~\ref{axiom:DC-M} holds
for the family of choices $C_{M}^{\regret,\cP^+ |^{\chi} E}$ if and
only if $\cP^+$ is $\chi$-rectangular. 
\erethm

We prove the two directions of implication in the theorem separately.
Note that the proof that $\chi$-rectangularity implies
Axiom~\ref{axiom:DC-M} does not require $C(\cP^+)$ to be convex. 

\begin{claim}
If $\cP^+$ is $\chi$-rectangular, then Axiom~\ref{axiom:DC-M} holds for the family of choices $C_{M}^{\regret,\cP^+ |^{\chi} E}$.
\end{claim}
\begin{proof}
By Theorem~\ref{thm:equivalence}, it suffices to show that SEP holds.
For the first part of SEP, we must show that
\begin{align}\label{equ:SEPCondRect}
\regret_M^{\cP^+|^{\chi}F} (f) = \sup_{(\Pr,\alpha_{}) \in \cP^+|^{\chi}F}
  \alpha_{} \left( \Pr(E\cap F) \regret_{M}^{\cP^+|^{\chi}(E \cap
    F)}(f) + \Pr(E^c \cap F) \regret_{M}^{\cP^+|^{\chi}(E^c \cap
    F)}(f)\right). 
\end{align}

\commentout{
In the case of likelihood updating ($|^l$), we have to check the case
where $\ucP(E \cap F) = 0$ or $\ucP(E^c \cap F) = 0$. 
If $\ucP(E \cap F) = 0$, then $\ucP(E^c \cap F) = 1$ and (\ref{equ:SEPCondRect}) reduces to
$$\begin{array}{ll}
& \regret_M^{\cP^+|^lF} (f) = \sup_{(\Pr,\alpha) \in \cP^+|^lF} \alpha \Pr(E^c \cap F) \sup_{(\Pr_1,\alpha_1) \in \cP^+|^l E^c \cap F}\alpha_{1} \sum_{s \in E^c \cap F} {\Pr}_1(s) \regret_{M}(f,s) \\
=&\regret_M^{\cP^+|^lF} (f) = \sup_{(\Pr_1,\alpha_1) \in \cP^+|^l E^c \cap F}\alpha_{1} \sum_{s \in E^c \cap F} {\Pr}_1(s) \regret_{M}(f,s),
\end{array}$$
which is just the definition of regret.
Similarly, if $\ucP(E^c \cap F) = 0$, then $\ucP(E\cap F) = 1$, and
(\ref{equ:SEPCondRect}) reduces to 
$$\begin{array}{ll}
& \regret_M^{\cP^+|^lF}(f) = \sup_{(\Pr,\alpha) \in \cP^+|^lF} \alpha \Pr(E\cap F) \sup_{(\Pr_1,\alpha_1) \in \cP^+|^l E \cap F}\alpha_{1} \sum_{s \in E \cap F} \Pr_1(s) \regret_{M}(f,s)) \\
=&\regret_M^{\cP^+|^lF}(f) = \sup_{(\Pr_1,\alpha_1) \in \cP^+|^l E\cap F}\alpha_{1} \sum_{s \in E\cap F} \Pr_1(s) \regret_{M}(f,s)),
\end{array}$$
which is again just the definition of regret.

We continue with the general $|^\chi$ updating.
} 
Unwinding the definitions, 
(\ref{equ:SEPCondRect}) is equivalent to 
$$\begin{array}{ll}
& \regret_M^{\cP^+|^{\chi}F} (f) \\
= & \sup_{(\Pr_3,\alpha_3) \in \cP^+|^{\chi}F} \alpha_{\Pr_3} \left(
\Pr_3(E\cap F) \sup_{(\Pr_1,\alpha_1) \in
  \cP^+|^{\chi}F}{\alpha_{1,E\cap F}^{\chi}} \sum_{s \in E\cap F}
\Pr_1(s|(E \cap F)) \regret_{M}(f,s))  \right. \\ 
& \left. + \Pr_3(E^c \cap F) \sup_{(\Pr_2,\alpha_2) \in \cP^+|^{\chi}F} \alpha_{2,E^c\cap F}^{\chi} \sum_{s\in E^c \cap F} \Pr_2(s|(E^c \cap F)) \regret_{M}(f,s)) \right)
.\end{array}$$
The $\sup$s in this expression are taken on by some
$(\Pr_1^*,\alpha_1^*),(\Pr_2^*,\alpha_2^*),(\Pr_3^*,\alpha_3^*) \in
\overline{\cP^+|^{\chi}F}$.  
%
By $\chi$-rectangularity, 
we have that for all $(\Pr_1,\alpha_1),(\Pr_2,\alpha_2),(\Pr_3,\alpha_3) \in \cP^+|^{\chi}F$,
\begin{align}\label{equ:rectangularity1}
\alpha_{{\Pr}_3} {{\Pr}_3(E \cap F)}{\alpha_{1,E\cap
    F}^{\chi}}{\Pr}_1|(E \cap F) + \alpha_{{\Pr}_3}{{\Pr}_3(E^c \cap
  F)} {\alpha_{2,E^c \cap F}^{\chi}}{\Pr}_2|(E^c \cap F) 
 \in \overline{C(\cP^+|^{\chi}F)}.\end{align} 
Thus, for all $\epsilon > 0$, 
$$\begin{array}{ll}
& \regret_M^{\cP^+|^{\chi}F}(f)\\
= &\regret_{M}^{C(\cP^+|^{\chi}F)}(f) \ \ \ \ \mbox{[by Lemma~\ref{lemma:C}]}\\ 
\geq & \alpha_{3}^* \left( \Pr_3^*(E\cap F) {(\alpha^*_{1,E\cap F})^{\chi}}
\sum_{s \in E\cap F} \Pr_1^*(s|(E \cap F)) \regret_{M}(f,s))
\right. \\  
& \ \ \ \ \left. + \Pr_3^*(E^c \cap F) {(\alpha^*_{2,E^c\cap F})^{\chi}}
\sum_{s\in E^c \cap F} \Pr_2^*(s|(E^c \cap F)) \regret_{M}(f,s)) \right) - \epsilon
\ \mbox{[by (\ref{equ:rectangularity1})]}.
\end{array}$$
Therefore, 
$$\begin{array}{ll}
& \regret_M^{\cP^+|^{\chi}F}(f)\\
\geq & \alpha_{3}^* \left( \Pr_3^*(E\cap F) {(\alpha^*_{1,E\cap F})^{\chi}}
\sum_{s \in E\cap F} \Pr_1^*(s|(E \cap F)) \regret_{M}(f,s))
\right. \\  
& \ \ \ \ \left. + \Pr_3^*(E^c \cap F) {(\alpha^*_{2,E^c\cap F})^{\chi}}
\sum_{s\in E^c \cap F} \Pr_2^*(s|(E^c \cap F)) \regret_{M}(f,s)) \right) \\
= &\sup_{(\Pr_3,\alpha_3) \in \cP^+|^{\chi}F} \alpha_{3} \left( 
\Pr_3(E\cap F) \sup_{(\Pr_1,\alpha_1) \in
\cP^+|^{\chi}F}{\alpha_{1,E\cap F}^{\chi} 
 } \sum_{s \in E\cap F} \Pr_1(s|(E \cap
F)) \regret_{M}(f,s))  \right. \\
& \ \ \ \ \left. + \Pr_3(E^c \cap F) \sup_{(\Pr_2,\alpha_2) \in \cP^+|^{\chi}F}{\alpha_{2,E^c \cap F}^{\chi} } \sum_{s\in E^c \cap F} \Pr_2(s|(E^c \cap F)) \regret_{M}(f,s)) \right)  \\
& \ \ \ \ \ \mbox{[by the choice of $(\Pr^*_i,\alpha^*_i)$, $i = 1,2,3$]} \\
= & \sup_{(\Pr,\alpha) \in \cP^+|^{\chi}F} \alpha \left( \Pr(E\cap F) \regret_{M}^{\cP^+|^{\chi}(E \cap F)}(f) + \Pr(E^c \cap F) \regret_{M}^{\cP^+|^{\chi}(E^c \cap F)}(f)\right),
\end{array}$$
as required.

It remains to show the opposite inequality in (\ref{equ:SEPCondRect}),
namely, that  
$$\regret_M^{\cP^+|^{\chi}F} (f) \le \sup_{(\Pr,\alpha) \in \cP^+|^{\chi}F} \alpha \left(
\Pr(E\cap F) \regret_{M}^{\cP^+|^{\chi}(E \cap F)}(f) + \Pr(E^c \cap F)
\regret_{M}^{\cP^+|^{\chi}(E^c \cap F)}(f)\right).$$
It suffices to note that the right-hand side is equal to 
$$\begin{array}{ll}
& \sup_{(\Pr,\alpha) \in \cP^+|^{\chi}F} \left( {\alpha}\Pr(E \cap F) \sup_{(\Pr_1,\alpha_1) \in \cP^+|^{\chi} F}{\alpha_{1,E\cap F}^{\chi} } \sum_{s \in E\cap F} \Pr_1(s|E \cap F)
\regret_{M}(f,s))  \right. \\ 
& \left. + {\alpha} \Pr(E^c \cap F ) \sup_{(\Pr_2,\alpha_2) \in
    \cP^+|^{\chi} F} \alpha_{2,E^c \cap F}^{\chi} 
  \sum_{s\in E^c \cap F} \Pr_2(s|E^c \cap F) \regret_{M}(f,s)) \right)\\
\geq & \overline{E}_{\cP^+|^{\chi}F}(\regret_M(f))\ \ \ \mbox{[by
    rectangularity]} \\ 
= & \regret_M^{\cP^+|^{\chi}F}(f) . 
\end{array}$$
This completes the proof that (\ref{equ:SEPCondRect}) holds.

\commentout{
For the second part of SEP, 
if
$\regret_{M}^{\cP^+|^{\chi}(E \cap F)}(f) \neq 0$, by
the first part of SEP we have that 
$$\begin{array}{ll}
&\regret_{M}^{\cP^+|^{\chi}F} (f) \\ 
= & \sup_{(\Pr,\alpha) \in \cP^+|^{\chi}F} \alpha
\left( \Pr(E\cap F) \regret_{M}^{\cP^+|^{\chi}(E \cap F) }(f) + \Pr(E^c \cap F) \regret_{M}^{\cP^+|^{\chi}(E^c \cap F)}(f)\right).
\end{array}$$
We claim that the $\sup$ must be taken on by some $(\Pr,\alpha) \in
\cP^+|^{\chi}F$. Suppose, for the purpose of contradiction, that it
is not.  
Then since $C(\cP^+|^{\chi}F)$ is closed by assumption, the $\sup$ is
taken on by some $\alpha\Pr \in C(\cP^+|^{\chi}F) \backslash \{
\alpha\Pr : (\Pr,\alpha) \in \cP^+|^{\chi}F\}$. 
This means that for some $p \in \{ p : p \leq \alpha \Pr,
(\Pr,\alpha) \in \cP^+|^{\chi}F \} = C(\cP^+|^{\chi}F)$, we have that  
$$\begin{array}{ll}
& p(E\cap F) \regret_{M}^{\cP^+|^{\chi}(E \cap F) }(f) + p(E^c \cap
F) \regret_{M}^{\cP^+|^{\chi}(E^c \cap F)}(f) \\ 
> &  \alpha\Pr(E\cap F) \regret_{M}^{\cP^+|^{\chi}(E \cap F) }(f) +
\alpha\Pr(E^c \cap F) \regret_{M}^{\cP^+|^{\chi}(E^c \cap F)}(f), 
\end{array}$$ for all $(\Pr,\alpha) \in \cP^+|^{\chi}F$.
However, this is impossible. Therefore, the $\sup$ is taken on by
some $(\Pr^*,\alpha^*) \in \cP^+|^{\chi}F$. 

We now return to the proof that the second part of SEP holds.
By $\chi$-rectangularity, for all $(\Pr,\alpha) \in \cP^+$, if $\alpha \Pr(E\cap F) = 0$, then $ \alpha \Pr( E^c \cap F ) < \sup_{(\Pr',\alpha') \in \cP^+|^{\chi}F}\alpha' \Pr'(E^c \cap F)$.
Therefore, for all $(\Pr^*, \alpha^*) \in \cP^+$, either $\alpha^* \Pr^*(E\cap F) > 0$, or $(\Pr^*, \alpha^*)$ cannot take on the $sup$ in $\sup_{(\Pr,\alpha) \in \cP^+|^{\chi}F} \alpha \Pr(E^c\cap F) \regret_{M}^{\cP^+|^{\chi}(E^c \cap F)}(f)$, and therefore we have
$$\begin{array}{ll}
&\regret_{M}^{\cP^+|^{\chi}F} (f) \\ 
= & \sup_{(\Pr,\alpha) \in \cP^+|^{\chi}F} \alpha
\left( \Pr(E\cap F) \regret_{M}^{\cP^+|^{\chi}(E \cap F) }(f) + \Pr(E^c \cap F) \regret_{M}^{\cP^+|^{\chi}(E^c \cap F)}(f)\right) \mbox{[by (\ref{equ:SEPCondRect})]}\\
> &   \sup_{(\Pr,\alpha) \in \cP^+|^{\chi}F} \alpha \Pr(E^c\cap F) \regret_{M}^{\cP^+|^{\chi}(E^c \cap F)}(f),
\end{array}$$
as required by SEP.}
For the second part of SEP, suppose that $\overline{\cP}^+(E\cap F) > 0$ and 
$\regret_{M}^{\cP^+|^{\chi}(E \cap F)}(f) \neq 0$.
If $\regret_{M}^{\cP^+|^{\chi}(E^c \cap F)}(f) =
0$ then, since $\overline{\cP}^+(E\cap F) > 0$, we have that  
$\regret_{M}^{\cP^+|^{\chi} F}(f) > 0 = \sup_{(\Pr,\alpha) \in
  \cP^+|^{\chi}F} \alpha Pr(E^c \cap F) \regret_{M}^{\cP^+|^{\chi}(E^c
  \cap F)}(f)$, as desired. 
Otherwise, by part (b) of $\chi$-rectangularity, 
for all $\delta > 0$, there exists $(\Pr,\alpha) \in \cP^+|^{\chi} F$ such that
$\alpha( \delta \Pr(E \cap F) + \Pr(E^c \cap F))  >
\sup_{(\Pr',\alpha') \in \cP^+} \alpha' \Pr'(E^c \cap F)$. 
Therefore, using the first part of SEP, we have
$$\begin{array}{ll}
&\regret_{M}^{\cP^+|^{\chi}(E \cap F)}(f)\\
= &\sup_{(\Pr,\alpha) \in \cP^+|^{\chi}F} \alpha
\left( \Pr(E\cap F) \regret_{M}^{\cP^+|^{\chi}(E \cap F) }(f) +
\Pr(E^c \cap F) \regret_{M}^{\cP^+|^{\chi}(E^c \cap F)}(f)\right)\\
= & \regret_{M}^{\cP^+|^{\chi}(E^c \cap F)}(f) \sup_{(\Pr,\alpha) \in \cP^+|^{\chi}F} \alpha 
\left( \Pr(E\cap F) \frac{\regret_{M}^{\cP^+|^{\chi}(E \cap F) }(f)}{\regret_{M}^{\cP^+|^{\chi}(E^c \cap F)}(f)} +
\Pr(E^c \cap F) \right)\\
> & \regret_{M}^{\cP^+|^{\chi}(E^c \cap F)}(f) \sup_{(\Pr,\alpha)
  \in \cP^+|^{\chi}F} \alpha Pr(E^c \cap F) \ \ \ 
\mbox{[by part (b) of $\chi$-rectangularity]}\\ 
= & \sup_{(\Pr,\alpha) \in \cP^+|^{\chi}F} \alpha Pr(E^c \cap F) \regret_{M}^{\cP^+|^{\chi}(E^c \cap F)}(f),
\end{array}$$ as required.
%
\end{proof}

\begin{claim}
If $C(\cP^+)$ is convex and Axiom~\ref{axiom:DC-M} holds
for the family of choices $C_{M}^{\regret,\cP^+ |^{\chi} E}$, then
$\cP^+$ is $\chi$-rectangular. 
\end{claim}
\begin{proof}
Suppose that $\chi$-rectangularity does not hold. 
Then one of the three conditions of rectangularity must fail.  

First suppose that it is (a); that is, for some 
$(\Pr_1,\alpha_1), (\Pr_2,\alpha_2), (\Pr_3,\alpha_3) \in \cP^+$, we
have $ \Pr_1( E\cap F) > 0 $ and $\Pr_2(E^c \cap F) > 0$ and 
$$
\alpha_3 {\Pr}_3( E \cap F) \alpha^{\chi}_{1,E\cap F} {\Pr}_1|(E \cap F) + \alpha_3 {\Pr}_3(E^c\cap F)
\alpha^{\chi}_{2, E^c \cap F} {\Pr}_2|(E^c\cap F) \notin \overline{C(\cP^+)|^{\chi} F}. 
$$
Let $p^* = \alpha_3 {\Pr}_3( E \cap F)
\alpha_{1,E\cap F}^{\chi} {\Pr}_1| (E \cap F) + \alpha_3
      {\Pr}_3(E^c\cap F) \alpha_{2,E^c \cap F}^{\chi} {\Pr}_2|(E^c\cap
      F)$.
Since we have assumed that $C(\cP^+)$ is convex, we have
that $C(\cP^+|^{\chi} F)$ is also convex.
By Lemma~\ref{lem:hyperplane},
there exists a non-negative vector $\theta$ such that for
all $\alpha\Pr \in \overline{C(\cP^+|^{\chi}F)}$, we have 
$$\sum_{s \in F} \alpha\Pr(s) \theta(s) < \sum_{s\in F} p^*(s) \theta(s).$$

We construct a decision problem $D$ based on $(S,\Sigma)$.
$D$ has two stages: in the first stage, nature chooses a state $s \in
S$, but only states in $F\subseteq S$ are chosen with positive
probability, so when the DM plays, his beliefs are characterized by
$\cP^+|^{\chi}F$.
In the second stage, the DM chooses an action from the set $M = \{f,g\}$, with utilities defined as follows:
$$\begin{array}{ll}
& u(f,s) = -\theta({s}), \text{ and } \\
& u(g,s) = 0 \text{ for all s}.
\end{array}$$
The act $f$ will have regret precisely $\theta(s)$ in state $s\in S$.
By Lemma~\ref{lem:hyperplane},
$$\begin{array}{ll}
& \sup_{(\Pr,\alpha) \in \cP^+} \alpha \left( \Pr(E \cap F)
\regret_{M}^{\cP^+|^{\chi}(E \cap F)}(f)
+ \Pr(E^c \cap F) \regret_{M}^{\cP^+|^{\chi}(E^c \cap F)}(f)\right) \\ 
\geq & \alpha_{\Pr_3} \left( \Pr_3(E \cap F)
\regret_{M}^{\alpha_{1,E \cap F}^{\chi} \Pr_1|(E \cap F)}(f)
+ \Pr(E^c \cap F) \regret_{M}^{\alpha_{2,E^c \cap F}^{\chi}\Pr_2|(E^c \cap F)}(f)\right) \\ 
= & \sum_{s\in F} p^*(s) \theta(s) \\
> & \sup_{(\Pr,\alpha) \in \cP^+|^{\chi}F}
\regret_{M}^{\cP^+|^{\chi}F}(f),
\end{array}$$ violating SEP.
By Theorem~\ref{thm:equivalence}, Axiom~\ref{axiom:DC-M} cannot hold.

Now suppose that condition (b) in rectangularity does not hold.
That is, for some $\delta > 0$, for all $(\alpha,\Pr) \in \cP^+$,
$\alpha( \delta \Pr(E \cap F) +  \Pr (E^c \cap F)  ) \le
\sup_{(\Pr',\alpha') \in \cP^+} \alpha' \Pr'(E^c \cap F)$. 
We construct a decision problem $D$ based on $(S,\Sigma)$.
$D$ has two stages: in the first stage, nature chooses a state $s \in S$. 
In the second stage, the DM chooses an action from the set $M = \{f,g\}$,
with utilities defined as follows:
$$\begin{array}{lll}
u(f,s) = 0  &\mbox{for all }s \in S, \\
u(g,s) = 
-\delta  &\mbox{if } s \in E\cap F\\
u(g,s) = -1  &\mbox{if $s \notin E \cap F$.}
\end{array}$$
%
Then we have that $\regret_{M}^{\cP^+|^{\chi}(E \cap F)}(g) = \delta$ and $\regret_{M}^{\cP^+|^{\chi}(E^c \cap F)}(g) = 1$. 
Using SEP and the choice of $\delta$, we must have 
$$\begin{array}{lll}
  \regret_{M}^{\cP^+|^{\chi}F} (g) &= &\sup_{(\Pr,\alpha) \in \cP^+|^{\chi}F}
\alpha (\Pr(E\cap F) \delta + \Pr(E^c \cap F) )
\\
&\le &\sup_{\Pr \in \cP^+|^{\chi}F} \alpha \Pr(E^c \cap F)
\regret_{{M}}^{\cP^+|^{\chi}(E^c \cap F)}(g).
\end{array} 
$$
Clearly, 
$$ \regret_{M}^{\cP^+|^{\chi}F} (g) \ge \sup_{(\Pr,\alpha) \in \cP^+|^{\chi}F}
\alpha  \Pr(E^c \cap F) \regret_{{M}}^{\cP^+|^{\chi}(E^c \cap F)}(g).$$ 
Thus, 
$$ \regret_{M}^{\cP^+|^{\chi}F} (g) = \sup_{(\Pr,\alpha) \in \cP^+|^{\chi}F}
\alpha  \Pr(E^c \cap F) \regret_{{M}}^{\cP^+|^{\chi}(E^c \cap F)}(g),$$ 
violating the second condition of SEP.
Therefore, by Theorem~\ref{thm:equivalence}, Axiom~\ref{axiom:DC-M}
does not hold.  

Finally, suppose that condition (c) in rectangularity does not hold.
Then for some nonnegative real vector $\theta \in \R^{|S|}$, 
\begin{equation}\label{eq:SEPc}
\begin{array}{ll}
&\sup_{(\Pr,\alpha) \in \cP^+|^{\chi}F} \left( {\alpha} \Pr(E) \sup_{(\Pr_1,\alpha_1) \in \cP^+|^{\chi}(E \cap F)} \sum_{s \in E\cap F} \alpha_1 \Pr_1(s|E)
\theta(s))  \right. \\ 
& \left. + {\alpha} \Pr(E^c) \sup_{(\Pr_2,\alpha_2) \in
    \cP^+|^{\chi}(E^c \cap F)}  \sum_{s\in E^c \cap F} \alpha_2
\Pr_2(s|E^c) \theta(s)) \right)\\ 
< & \sup_{(\Pr,\alpha) \in \cP^+|^{\chi}F}{\alpha} \sum_{s \in F}
\Pr(s) \theta(s). 
\end{array}
\end{equation}
We construct a decision problem $D$ based on $(S,\Sigma)$.
$D$ has two stages: in the first stage, nature chooses a state $s \in S$. 
In the second stage, the DM chooses an action from the set $M = \{f,g\}$, with utilities defined as follows:
$$\begin{array}{ll}
u(g,s) = -\theta(s) \ \ \ \mbox{ for all } s \in S.
\\
u(f,s) = 0 \ \ \ \mbox{ for all } s \in S.
\end{array}$$
So we have
$$\begin{array}{ll}
&\sup_{(\Pr,\alpha) \in \cP^+|^pF} \alpha \left( \Pr(E\cap F)
  \regret_{M}^{\cP^+|^p(E \cap F)}(g) + \Pr(E^c \cap F)
  \regret_{M}^{\cP^+|^p(E^c \cap F)}(g)\right)\\
= & \sup_{(\Pr,\alpha) \in \cP^+|^{\chi}F} \alpha \left( \Pr(E \cap F ) \overline{E}_{\cP^+|^{\chi}(E \cap F)}(\theta) + \Pr(E^c\cap F) \overline{E}_{\cP^+|^{\chi}(E^c \cap F)}(\theta) \right)\\
< & \overline{E}_{\cP^+|^{\chi}F}(\theta) \ \ \ \mbox{[by (\ref{eq:SEPc})]}\\  
= & \regret_M^{\cP^+|^pF}(g).
\end{array}$$
This means that SEP, and hence Axiom 1, is violated, a contradiction.
\end{proof}

\fullv{
\bibliographystyle{chicagor}
}
\shortv{
\bibliographystyle{abbrv}
\vspace{-17pt}
}
\bibliography{joe,awareness}

\end{document}